\documentclass[sigconf, nonacm]{acmart}

\pdfoutput=1

\newcommand\vldbdoi{XX.XX/XXX.XX}

\newcommand\vldbpagestyle{plain}

\theoremstyle{definition}
\newtheorem{definition}{Definition}

\newtheorem{theorem}{Theorem}[section]
\renewcommand{\proofname}{\textbf{Proof}}

\usepackage{multirow}
\usepackage{array}
\usepackage{stackengine}
\usepackage{makecell}
\usepackage{threeparttable}
\usepackage{longtable}

\usepackage[bottom]{footmisc}

\begin{document}
\title{Revisiting CNNs for Trajectory Similarity Learning}


\author{Zhihao Chang}
\affiliation{%
  \institution{Zhejiang University, China}
}
\email{changzhihao@zju.edu.cn}

\author{Linzhu Yu}
\affiliation{%
  \institution{Zhejiang University, China}
}
\email{linzhu@zju.edu.cn}

\author{Huan Li}
\affiliation{%
  \institution{Zhejiang University, China}
}
\email{lihuan.cs@zju.edu.cn}

\author{Sai Wu}
\affiliation{%
  \institution{Zhejiang University, China}
}
\email{wusai@zju.edu.cn}

\author{Gang Chen}
\affiliation{%
  \institution{Zhejiang University, China}
}
\email{cg@zju.edu.cn}

\author{Dongxiang Zhang}
\affiliation{%
  \institution{Zhejiang University, China}
}
\email{zhangdongxiang@zju.edu.cn}

\begin{abstract}

Similarity search is a fundamental but expensive operator in querying trajectory data, due to its quadratic complexity of distance computation. To mitigate the computational burden for long trajectories, neural networks have been widely employed for similarity learning and each trajectory is encoded as a high-dimensional vector for similarity search with linear complexity. Given the sequential nature of trajectory data, previous efforts have been primarily devoted to the utilization of RNNs or Transformers.

In this paper, we argue that the common practice of treating trajectory as sequential data results in excessive attention to capturing long-term global dependency between two sequences. Instead, our investigation reveals the pivotal role of local similarity, prompting a revisit of simple CNNs for trajectory similarity learning. We introduce ConvTraj, incorporating both 1D and 2D convolutions to capture sequential and geo-distribution features of trajectories, respectively. In addition, we conduct a series of theoretical analyses to justify the effectiveness of ConvTraj. Experimental results on four real-world large-scale datasets demonstrate that ConvTraj achieves state-of-the-art accuracy in trajectory similarity search. Owing to the simple network structure of ConvTraj, the training and inference speed on the Porto dataset with 1.6 million trajectories are increased by at least $240$x and $2.16$x, respectively. The source code and dataset can be found at \textit{\url{https://github.com/Proudc/ConvTraj}}.
\end{abstract}

\maketitle


\pagestyle{\vldbpagestyle}



\section{Introduction}
Trajectory similarity plays a fundamental role in numerous trajectory analysis tasks. Numerous distance measures, such as Discrete Frechet Distance (DFD)~\cite{DBLP:journals/ijcga/AltG95}, the Hausdorff distance~\cite{DBLP:journals/tits/AtevMP10}, Dynamic Time Warping (DTW)~\cite{DBLP:conf/icde/YiJF98}, and Edit Distance on Real sequence (EDR)~\cite{DBLP:conf/sigmod/ChenOO05}, have been proposed and employed in a wide spectrum of applications, including but not limited to trajectory clustering~\cite{DBLP:conf/www/ChanGS18,DBLP:conf/pods/AgarwalFMNPT18}, anomaly detection~\cite{DBLP:journals/tkde/ZhangCWYTC22,DBLP:journals/pami/LaxhammarF14}, and similar retrieval~\cite{DBLP:journals/pvldb/XieLP17,DBLP:conf/sigmod/Shang0B18}.

Generally speaking, these distance measures involve the optimal point-wise alignment between two trajectories. The distance calculation often relies on dynamic programming and incurs quadratic computational complexity. This limitation poses a significant constraint, particularly when confronted with large-scale datasets with long trajectories. In recent years, trajectory similarity learning has emerged as the mainstream approach to mitigate the computational burden. The main idea is to encode each trajectory sequence $T_i$ into a high-dimensional vector $V_i$ such that the real distance between $T_1$ and $T_2$ can be approximated by the distances between their derived vectors $V_1$ and $V_2$. Consequently, the complexity of distance calculation can be reduced from quadratic to linear.

\begin{figure}[htb]
  \centering
  \includegraphics[width=\linewidth]{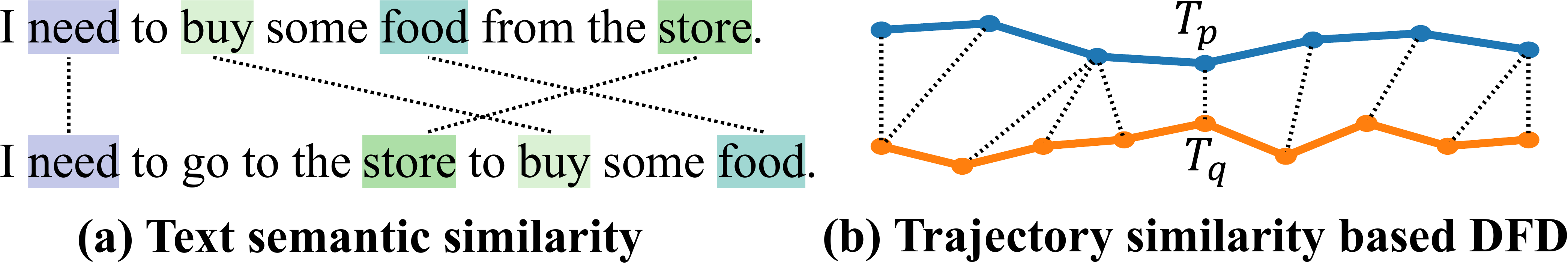}
  \caption{Texts feature intercrossed matching pairs, whereas trajectories do not.}
  \label{fig:alignment}
\end{figure}

Given the sequential nature of trajectory data, existing methods for trajectory similarity learning can be categorized into RNN-based or Transformer-based. RNN-based methods, including  NeuTraj~\cite{DBLP:conf/icde/YaoCZB19}, Traj2SimVec~\cite{DBLP:conf/ijcai/ZhangZJZSSW20}, and T3S~\cite{DBLP:conf/icde/YangW0Q0021},  employ RNN or its variants (e.g, GRU~\cite{DBLP:conf/emnlp/ChoMGBBSB14}, LSTM~\cite{DBLP:journals/neco/HochreiterS97}) as the core encoder, which can be augmented with additional components such as spatial attention memory in NeuTraj and point or structure matching mechanisms in Traj2SimVec and T3S to enhance performance. Due to the success of Transformer in NLP, 
TrajGAT~\cite{DBLP:conf/kdd/YaoHDCHB22} and TrajCL~\cite{DBLP:conf/icde/Chang0LT23} adopt Transformer to learn trajectory embedding, which can effectively capture the long-term dependency of sequences. 
\begin{table*}
\caption{Performance of Transformer with different attention window sizes. We report the hit rates for two measures: DFD and DTW. The dataset includes 6000 items selected from Porto, with 3000 for training, 1000 for query, and 2000 as the candidate set.}
\label{table:motivating-example}
\begin{threeparttable}
\resizebox{0.9\textwidth}{!}{%
\begin{tabular}{c|c|c|c|cccc|cccc}
\toprule
& & & & \multicolumn{4}{c|}{\textbf{DFD}} & \multicolumn{4}{c}{\textbf{DTW}} \\
Method&  \# Paras &\makecell{(Train time Per\\ Epoch) * \# Epochs} &\makecell{Inference\\time}  & HR@1 & HR@5 & HR@10 & HR@50  & HR@1 & HR@5 & HR@10 & HR@50 \\ \midrule
global attention& 3.38M & 17.28s * 1000 & 3.58s &   $22.10$   &  $32.58 $    & $39.11 $      &   $50.11 $    &   $29.40 $   &   $46.22 $   &     $54.60 $  &   $63.41 $         \\  
local attention ($w = 10$) & 3.38M & 17.28s * 1000 & 3.58s&   $\textbf{23.20} $   &  $\textbf{36.74} $    & $\textbf{42.80} $      &   $\textbf{54.70} $    &      $31.50 $   &   $\textbf{48.60} $   &     $\textbf{54.92} $  &   $\textbf{65.06} $     \\
local attention ($w = 5$)& 3.38M & 17.28s * 1000 & 3.58s &   $21.80 $   &  $35.40 $    & $41.83 $      &   $54.42 $    &      $\textbf{33.60} $   &   $46.72 $   &     $52.34 $  &   $63.03 $ \\ 
\hline
1D CNN & 0.17M & 1.03s * 200 & 0.16s&   $\textbf{33.23} $  &  $\textbf{43.94} $    & $\textbf{50.84} $      &   $\textbf{64.78} $    &   $30.90 $   &   $46.66 $   &     $53.36 $  &   $\textbf{65.14} $         \\ 
\bottomrule
\end{tabular}%
}
\end{threeparttable}
\end{table*}

However, we argue that these common practices pay excessive attention to capturing long-term global dependency between two trajectories while ignoring point-wise similarity, which may potentially yield adverse effects. Instead, we should pay more attention to point-wise similarity in the local context. In support of this argument, we conducted an experiment on Porto\footnote{https://www.kaggle.com/competitions/pkdd-15-predict-taxi-service-trajectory-i/data} dataset to evaluate the effect of applying Transformer for trajectory encoding with different sizes of attention windows. The first variant is the original Transformer with global attention, where each token engages in self-attention by querying all other tokens. We also implemented two alternative variants with local attention, in which each token only queries its neighbors within a window of count $w$, i.e., the attention weights outside the window have been masked. 
We can observe from \autoref{table:motivating-example} that local attention has great potential to 
significantly outperform global attention. 
We explain that existing trajectory distance measurements are alignment-based and the edges for matching pairs are not intercrossed (as shown in \autoref{fig:alignment}). This property differs significantly from handling text data in NLP.

These observations reveal the pivotal role of \textit{local similarity}. Instead of adopting Transformer with masked local attention, we are interested in revisiting CNNs in the task of trajectory similarity learning. The reason is that CNNs can also well capture local similarity while offering the advantages of simplicity. As shown in ~\autoref{table:motivating-example}, with only 5\% of the parameters, a simple 1D CNN can remarkably outperform vanilla Transformers after convergence on the DFD. Although slightly lower than local attention on the DTW, 1D CNN has great advantages in efficiency. To further exploit the potential of CNNs, we present ConvTraj with two types of convolutions. We first use 1D convolution to capture the sequential features of trajectories. Then we represent the trajectory as a single-channel binary image and use 2D convolution to capture its geo-distribution. Finally, these features are fused as complementary clues to capture trajectory similarity. To justify the effectiveness of ConvTraj, we conduct a series of theoretical analyses. We prove that 1D convolution and 1D max-pooling can preserve effective distance bounds after embedding, and trajectories located in distant areas yield large distances via 2D convolution, all of which play an important role in trajectory similarity recognition.

We conducted extensive experiments to evaluate the performance of ConvTraj on four real-world datasets. Experimental results show that ConvTraj achieves state-of-the-art accuracy for similarity retrieval on four commonly used similarity measurements, including DFD, DTW, Hausdorff, and EDR. 
Furthermore, ConvTraj is at least $240$x faster in training speed and $2.16$x faster in inference speed, when compared with methods based on RNN and Transformer on the Porto dataset containing 1.6 million trajectories.

Our contributions are summarized in the following:
\begin{itemize}
    \item We argue that trajectory similarity learning should pay more attention to local similarity.
    \item We present a simple and effective ConvTraj with two types of CNNs for trajectory similarity computation.
    \item We conduct some theoretical analysis to help justify why such a simple ConvTraj can perform well.
    \item Extensive experiments on four real-world large-scale datasets established the superiority of ConvTraj over state-of-the-art works in terms of accuracy and efficiency.
\end{itemize}


\section{Related Work}
\subsection{Heuristic Trajectory Similarity Measures}
Heuristic measures between trajectories are derived from the distance between matching point pairs, these measures fall into three categories: (1) \emph{Linear-based} methods ~\cite{DBLP:conf/fodo/AgrawalFS93,DBLP:conf/ssdbm/Chang0TMS21} only need scan trajectories once to calculate their similarity but may lead to sub-optimal point matches. (2) \emph{Dynamic programming-based} methods are proposed to tackle this issue, such as DTW~\cite{DBLP:conf/icde/YiJF98}, DFD~\cite{DBLP:journals/ijcga/AltG95}, and others~\cite{DBLP:conf/sigmod/ChenOO05, DBLP:conf/vldb/ChenN04, DBLP:conf/icde/VlachosGK02, DBLP:conf/icde/RanuPTDR15}.  
However, these measurements involve the optimal point-wise alignment between two trajectories without intercrossing between matching pairs and often incur quadratic complexity. Thus it poses significant challenges for similarity search from a large-scale dataset with long trajectories. (3) \emph{Enumeration-based} methods calculate all point-to-trajectory distance, i.e., the minimum distance between a point to any point on a trajectory, then aggregate it. For example, OWD~\cite{DBLP:journals/geoinformatica/LinS08} uses the average point-to-trajectory distance, while Hausdorff~\cite{DBLP:journals/tits/AtevMP10} uses the maximum.

\begin{figure*}
  \centering
  \includegraphics[width=\textwidth]{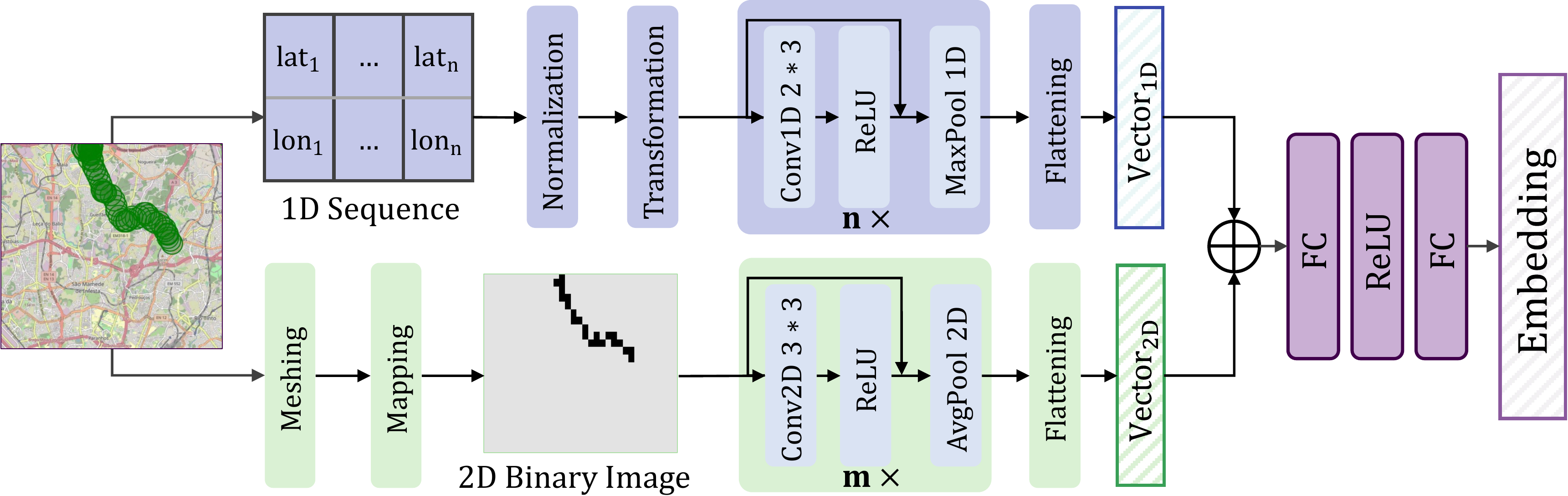}
  \caption{Input preprocessing and network structure of ConvTraj.}
  \label{fig:convtraj}
\end{figure*}

\subsection{Learning-based Trajectory Similarity}
In recent years, the field of trajectory similarity has witnessed a paradigm shift, primarily fueled by the progress in deep representation learning. This advancement has led to the development of numerous methodologies aimed at encoding trajectories into embedding spaces. Broadly, these approaches can be classified into three categories:
(1) \emph{Learn a model to approximate a measurement}. The purpose of these methods is to learn a neural network so that the distance in the embedding space can approximate the true distance between trajectories. Early attempts were generally based on recurrent neural networks, including NeuTraj~\cite{DBLP:conf/icde/YaoCZB19}, Traj2SimVec~\cite{DBLP:conf/ijcai/ZhangZJZSSW20}, T3S~\cite{DBLP:conf/icde/YangW0Q0021}, and TMN~\cite{DBLP:conf/icde/YangWLZQZ22}. Subsequently, some studies tried to capture the long-term dependency of trajectories based on Transformer~\cite{DBLP:conf/kdd/YaoHDCHB22, DBLP:conf/icde/Chang0LT23}.
(2) \emph{No given measurements are required to generate training signals}. These methods encode trajectories without the need to generate supervised signals based on measurements. Its purpose is to overcome the limitations of traditional measures such as non-uniform sampling rates and noise. Based on the network they use, these methods can be divided into RNN-based methods, including traj2vec~\cite{DBLP:journals/es/YaoZZHWHB18}, t2vec~\cite{DBLP:conf/icde/LiZCJW18}, E2DTC~\cite{DBLP:conf/icde/FangDCHGC21}, etc., CNN-based method TrjSR~\cite{DBLP:conf/ijcnn/CaoTWWX21}, and Transformer-based method TrajCL~\cite{DBLP:conf/icde/Chang0LT23}.
(3) \emph{Road networks-based methods.} There have been some studies on trajectory similarity based on road networks~\cite{DBLP:conf/kdd/HanWYS021,DBLP:conf/kdd/Fang0ZHCGJ22,DBLP:journals/www/ZhouHYCZ23,DBLP:conf/edbt/ChangTC023,DBLP:conf/aaai/Zhou0WS023}. These works use graph neural networks to encode road segments. Since such works introduce relevant knowledge from road networks, we consider them as different research directions and will not delve into these methods. 

TrjSR~\cite{DBLP:conf/ijcnn/CaoTWWX21} is a well-known CNN-based method for trajectory similarity. It maps trajectories into 2D images and uses super-resolution techniques. However, TrjSR loses the sequential features of trajectories, making it unable to differentiate between two trajectories with the same path but opposite directions. Our ConvTraj uses both 1D and 2D convolutions as the backbone and achieves better results.


\section{Problem Definition}
In this section, we present the definition of our research problem.

\begin{definition}[Trajectory]
A trajectory $T$ is a series of GPS points ordered based on timestamp $t$, and each point $p$ is a location in a two-dimensional geographic space containing latitude and longitude. Formally, a trajectory $T\in \mathbb{R}^{l \times 2}$ containing $l$ points can be expressed as $T = [p_1, p_2, p_3, ..., p_l]$, where $p_i=(p_i^{lat}, p_i^{lon})$ is the $i$-th location.
\end{definition}

\begin{definition}[Trajectory Measure Embedding]
Given a specific trajectory similarity measure $f(\cdot,\cdot)$, trajectory measure embedding aims to learn an approximate projection function $g$, such that for any pair of trajectories $T_i$ with $T_j$, the distance in the embedding space approximates the true distance between $T_i$ and $T_j$, i.e., $f(T_i, T_j) \approx d(g(T_i), g(T_j))$. Besides, the vectors in the embedding space should maintain the distance order of true distance, i.e., for any three trajectories $T_i$, $T_j$, and $T_k$, with $f(T_i, T_j) < f(T_i, T_k)$, we should ensure that $d(g(T_i), g(T_j)) < d(g(T_i), g(T_k))$. Here, $f(\cdot,\cdot)$ can be DFD, DTW, or any other measurements. At the same time, $d(\cdot,\cdot)$ is a measure between high-dimensional embedding vectors in the embedding space, such as Euclidean distance, Cosine distance, etc.

\end{definition}

\section{Methodology}

\subsection{Input Preprocessing}
Suppose there is a trajectory $T$ containing $l$ GPS points.
To process $T$ as the input of our ConvTraj, we perform the following two steps covering both one-dimensional and two-dimensional.

\textbf{One-dimensional Input.}
The input of our 1D convolution is a sequence, we thus treat the trajectory $T$ as a sequence with length $l$ and width 2 (i.e., latitude and longitude). For each point of $T$, we first normalize it using a min-max normalization function, and then apply a multi-layer perceptron (MLP) to perform a nonlinear transformation for each normalization point, thus the trajectory can be processed as a sequence $Seq_{1D}$.

\textbf{Two-dimensional Input.}
The input of our 2D convolution is a binary image, we thus perform the following substeps to generate such an image for each trajectory. Initially, we determine a minimum bounding rectangle (MBR) within a two-dimensional space, encapsulating all points of the whole trajectory dataset. Subsequently, the MBR is partitioned into equal-sized grids based on a predetermined hyperparameter width $\delta$. Then for each trajectory $T$, its coordinates are mapped onto the grid, and each pixel within the grid cell is assigned a binary value, which is 1 if the trajectory point falls within the grid cell and 0 otherwise. Thus each raw trajectory is converted into a single-channel binary image $BI_{2D}$. 

\subsection{ConvTraj Network Structure}
As shown in \autoref{fig:convtraj}, the ConvTraj consists of three submodules: 1D convolution, 2D convolution, and feature fusion. The 1D convolution extracts sequential features from the trajectory, while the 2D convolution captures its geo-distribution. The feature fusion module then combines these features for comprehensive analysis. Detailed descriptions of these submodules are provided below.

\textbf{One-dimensional Convolution.}
As shown in ~\autoref{fig:convtraj}, 1D convolution is stacked by $n$ residual blocks consisting of a 1D convolution layer, a non-linear ReLU layer, and a max-pooling layer. Each operation is performed on rows of $Seq_{1D}$. By default, the convolution kernel size is $2*3$, the number of channels is $32$, the pooling stride is $2$, and the number of stacking layers $n$ is determined by the maximum length of the trajectory in the dataset. In the end, the features of all channels are flattened into a vector $V_{1D}$. 

\textbf{Two-dimensional Convolution.}
2D convolution is also stacked by $m$ residual blocks consisting of a 2D convolution layer, a non-linear ReLU layer, and an average-pooling layer. Each operation is performed on the single-channel binary image $BI_{2D}$. By default, the convolution kernel size is $3*3$, the number of channels is $4$, the pooling stride is $2$, and the number of stacking layers $m$ is $4$. In the end, the features of all channels are flattened into a vector $V_{2D}$.

\textbf{Feature Fusion.}
After performing 1D and 2D convolution on the trajectory in parallel, we concatenate the resulting feature vectors and pass them through an MLP. This submodule combines the sequence order features ($V_{1D}$) extracted by 1D convolution with the geo-distribution features ($V_{2D}$) extracted by 2D convolution, providing comprehensive information for similarity recognition. The final embedding $V$ of the trajectory can be formalized as:
\begin{gather}    
    V = MLP([V_{1D}, V_{2D}]).
\end{gather}

\subsection{Training Pipeline}\label{sec-training-pipeline}
We employ the mainstream training pipeline as shown in~\autoref{fig:overview}, and its details are introduced below.

\textbf{Loss Function.}
As shown in \autoref{fig:overview}, we use the combination of triplet loss~\cite{weinberger2009distance,hermans2017defense} $L_{T}$ and MSE loss $L_{M}$ as our loss function. i.e.:
\begin{gather}    
    Loss = L_{T}(T_a, T_p, T_n) + L_{M}(T_a, T_p, T_n),
\end{gather}
where
\begin{gather}    
    L_{T} = max \{ 0, d(V_a, V_p) - d(V_a, V_n) - \eta \}, \\
    L_{M} = |d(V_a, V_p) - f(T_a, T_p)| + |d(V_a, V_n) - f(T_a, T_n)|,
\end{gather}
in which $(T_a, T_p, T_n)$ is a triplet, and $T_a$ is the anchor trajectory, $T_p$ is the positive trajectory that has a smaller distance to $T_a$ than the negative trajectory $T_n$. $V_a$, $V_p$ and $V_n$ are the high-dimensional vectors corresponding to $T_a$, $T_p$ and $T_n$ in the embedding space. $f(\cdot,\cdot)$ represents the true distance between trajectories, and $d(\cdot,\cdot)$ is the Euclidean distance~\cite{DBLP:conf/icde/YaoCZB19,DBLP:conf/kdd/YaoHDCHB22} between two vectors. Besides, $\eta$ is the margin in the triplet loss whose value is $\eta = f(T_a, T_p) - f(T_a, T_n)$.

\textbf{Triplet Selection Method.}
Many studies~\cite{DBLP:conf/icde/YaoCZB19,DBLP:conf/ijcai/ZhangZJZSSW20,DBLP:conf/icde/YangWLZQZ22} have proposed various strategies to select triplets for training, but these often bring additional training costs. In this paper, we use the simplest strategy to select triplets. We regard each trajectory in the training set as $T_{a}$ in turn. For each $T_{a}$, we randomly select two trajectories from its top-k neighbors (k=200 by default) and use the trajectory closer to the $T_{a}$ as $T_{p}$, and trajectory farther to $T_{a}$ as $T_{n}$. 

\textbf{Training Process.}
Figure \ref{fig:overview} is the overall training process of ConvTraj, where the blue hollow circle represents the location where the positive trajectory should be in the embedding space after training and the orange hollow circle represents the negative. During the training, for a trajectory (the anchor trajectory in the upper left corner of Figure \ref{fig:overview}), we first use the triplet selection method introduced above to select the positive trajectory (the blue trajectory in Figure \ref{fig:overview}) and negative trajectory (the orange trajectory in Figure \ref{fig:overview}), then these three trajectories are encoded using ConvTraj with shared parameters, and corresponding embedding vectors are obtained, which we call $V_a$ (green full circle), $V_p$ (blue full circle), $V_n$ (orange full circle). Since the loss function we use is a combination of triplet loss and MSE loss, we hope that the distance between anchor and positive in the embedding space is the same as the actual distance (i.e., pulling the blue full circle toward the blue hollow circle), and the distance between anchor and negative in embedding space is the same as the actual distance (i.e. pushing the orange full circle toward the orange hollow circle). 
\begin{figure}[h]
  \centering
  \includegraphics[width=\linewidth]{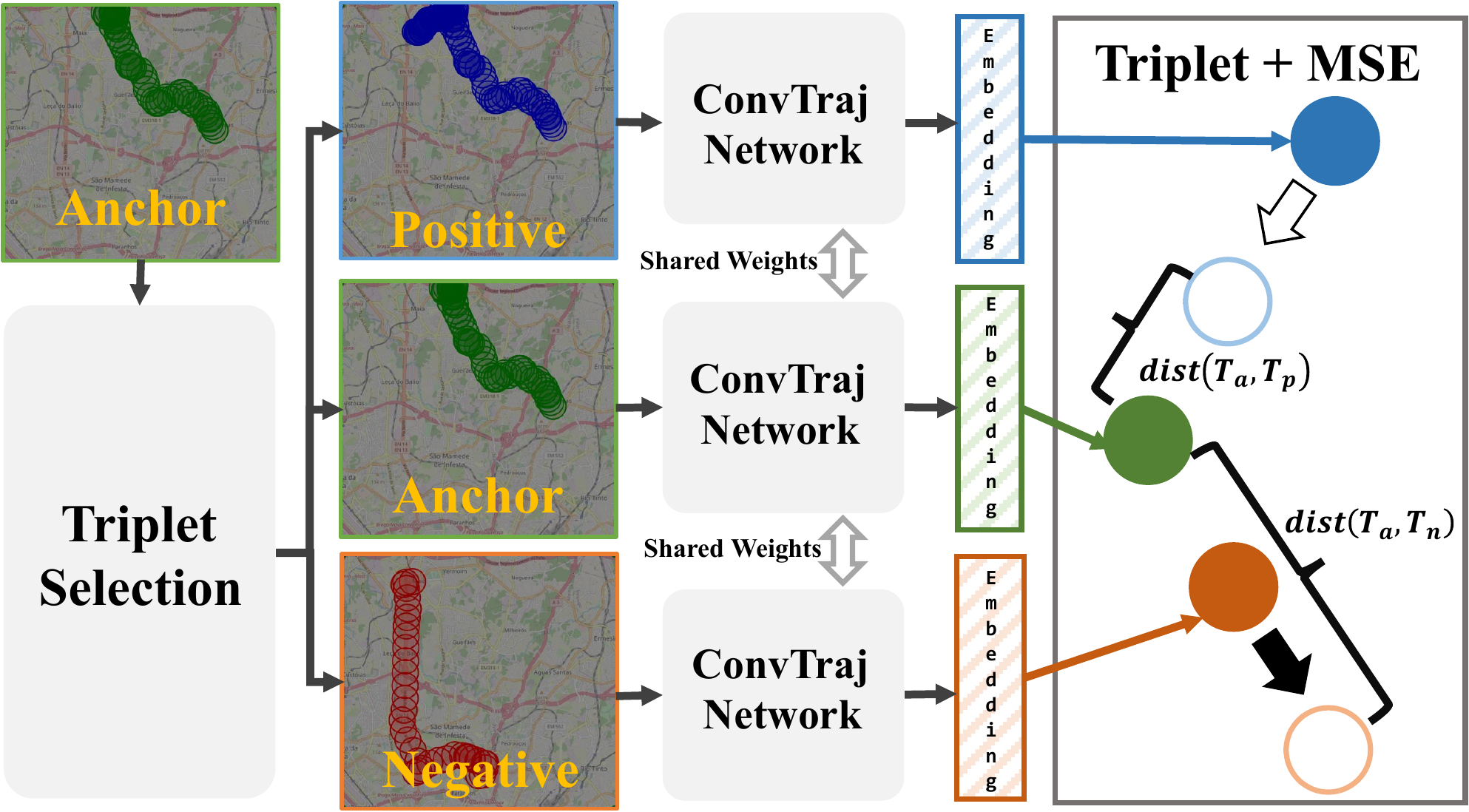}
  \caption{The training pipeline of ConvTraj.}
  \label{fig:overview}
\end{figure}
\section{Theoretical Analysis}
In this section, we will conduct some theoretical analysis from both 1D and 2D convolution to help justify why such a simple ConvTraj can perform well.
We take the DFD, which is widely used for trajectory similarity~\cite{DBLP:journals/tkde/ZhangCWYTC22,DBLP:journals/pvldb/XieLP17,DBLP:conf/edbt/ZhangTY19,DBLP:conf/edbt/TangYMW17}, as an example for analysis. In summary, we found that: (1) For a randomly initialized kernel of 1D convolution, the DFD between two trajectories can still be maintained to a large extent. (2) After 1D max-pooling, the DFD value has almost no change. (3) Trajectories located in distant areas not only have a large DFD value but also have a large Euclidean distance through 2D convolution. Since 1D convolution essentially rotates and scales the sequence and 2D convolution captures the geo-distribution of trajectories, thus similar conclusions can be easily generalized to other measurements. Basically, the analysis shows that 1D convolution and max-pooling can preserve the bounds for trajectory similarity learning, while 2D convolution can help capture the geo-distribution. This implies CNNs are a good choice in scenarios where trajectories need to be reduced in dimension or geo-distribution is required. This does not mean that RNNs or Transformers lack it, it is just difficult to analyze.

\subsection{Discrete Frechet Distance}
To facilitate understanding, we first present the formal definition of Discrete Frechet Distance:
\begin{definition}[Trajectory Coupling]
A coupling $L$ between two trajectories $T_1 = [p_1, p_2, ..., p_n]$ and $T_2 = [q_1, q_2, ..., q_m]$ is such a sequence of alignment:
\[ L=(p_{a_1}, q_{b_1}), (p_{a_2}, q_{b_2}), ..., (p_{a_t}, q_{b_t}),\]
where $a_1 = 1, b_1 = 1, a_t = n, b_t = m$. For all $i=1, ..., t$, we have $a_{i+1} = a_i$ or $a_{i+1} = a_i + 1$, and $b_{i+1} = b_i$ or $b_{i+1} = b_i + 1$.
\end{definition}

\begin{definition}[Discrete Frechet Distance]
Given two trajectories $T_1 = [p_1, p_2, ..., p_n]$ and $T_2 = [q_1, q_2, ..., q_m]$, the Discrete Frechet Distance $d_F$ between these     two trajectories is:
\[ d_F(T_1, T_2) = \min\limits_{L} \{ \max\limits_{(p_i, q_j)\in L} d(p_i, q_j)  \}, \]
where $L$ is an instance of coupling between $T_1$ and $T_2$, and $d(\cdot,\cdot)$ is Euclidean distance between two points.
\end{definition}

\subsection{One-dimensional Convolution}
\begin{definition}
Given a trajectory $X=
\big(\begin{smallmatrix}
x_{0, 0}&...&x_{0, M} \\
x_{1, 0}&...&x_{1, M} \\
\end{smallmatrix}\big)$ and a kernel $k=
\big(\begin{smallmatrix}
k_{0, 0}&...&k_{0, 2} \\
k_{1, 0}&...&k_{1, 2} \\
\end{smallmatrix}\big)$, we define the convolution operation of the $j_{th}$ point of $X$ with the kernel $k$ as:
\begin{align*}
c(X_j, k) =&\sum\nolimits_{m=0}^{1}(k_{m, 0} * x_{m, j-1} + k_{m, 1} * x_{m, j} + k_{m, 2} * x_{m, j+1}) \\
=&\sum\nolimits_{m=0}^{1}(k_{m, 0} * (x_{m, j} - \delta_{m, j}^{x}) + k_{m, 1} * x_{m, j} + k_{m, 2} \\ & * (x_{m, j} + \delta_{m, j + 1}^{x})),
\end{align*}
where $\delta_{i, j}^x = x_{i, j}-x_{i, j-1}$.
\end{definition}

\begin{theorem}[One-dimensional Convolution Bound]
\label{one-conv}
Given two trajectories $X=
\big(\begin{smallmatrix}
x_{0, 0}&...&x_{0, M} \\
x_{1, 0}&...&x_{1, M} \\
\end{smallmatrix}\big)$, $Y=
\big(\begin{smallmatrix}
y_{0, 0}&...&y_{0, N} \\
y_{1, 0}&...&y_{1, N} \\
\end{smallmatrix}\big)$, and $d_F(X, Y) = d_{xy}$. A one-dimensional convolution operation $C(\cdot)$ on $X$ and $Y$ with stride 1, padding 1, and kernel $k=
\big(\begin{smallmatrix}
k_{0, 0}&k_{0, 1}&k_{0, 2} \\
k_{1, 0}&k_{1, 1}&k_{1, 2} \\
\end{smallmatrix}\big)$
. We have:
\[ max(d(x_0^c, y_0^c), d(x_M^c, y_N^c)) \leq d_F(C(X), C(Y)) \leq \sqrt{S_0^2 + S_1^2} * d_{xy} + \Delta,\]
where $x_0^c=c(X_0, k),y_0^c=c(Y_0, k),x_M^c=c(X_M, k),y_N^c=c(Y_N, k)$,\\ $S_i=\sum_{j=0}^{2}k_{i, j}$, and  $\Delta=\max\limits_{0\leq i\leq M, 0\leq j\leq N} |\sum\nolimits_{m=0}^{1} k_{m, 0}*(\delta_{m, j}^{y}-\delta_{m, i}^{x}) + k_{m, 2}*(\delta_{m, i+1}^{x} - \delta_{m, j+1}^{y})|$.
\end{theorem}

\begin{proof}
Suppose that $C(X) = [x_0^c, ..., x_M^c]$, where $x_i^c = c(X_i, k)$.
Similarly, $C(Y) = [y_0^c, ..., y_N^c]$. Since $d_{F}(X,Y)=d_{xy}$, we assume that the coupling corresponding to $d_{xy}$ is $L^{*}$, and the indexs of $X$ and $Y$ in $L^{*}$ are $p$ and $q$. Then we apply $L^{*}$ to $d_{F}(C(X),C(Y))$, thus:
\begin{align*}
&d_F(C(X), C(Y))\leq\max\limits_{(x_i^c, y_j^c)\in L^{*}} d(x_i^c, y_j^c)\\
=&\max\limits_{(x_i^c, y_j^c)\in L^{*}} |c(X_i, k) - c(Y_j, k)|\\
=&\max\limits_{(x_i^c, y_j^c)\in L^{*}} |\sum\nolimits_{m=0}^{1} (k_{m, 0} * (x_{m, i} - \delta_{m, i}^{x} - y_{m, j} + \delta_{m, j}^{y}) \\&+ k_{m, 1} * (x_{m, i} - y_{m, j}) + k_{m, 2} * (x_{m, i} + \delta_{m, i + 1}^{x} - y_{m, j} - \delta_{m, j + 1}^{y}))|\\
=&\max\limits_{(x_i^c, y_j^c)\in L^{*}} |S_0 * (x_{0, i} - y_{0, j}) + S_1 * (x_{1, i} - y_{1, j}) + \Delta_{i, j}|\\
\leq&\max\limits_{(x_i^c, y_j^c)\in L^{*}} |S_0 * (x_{0, i} - y_{0, j}) + S_1 * (x_{1, i} - y_{1, j})| + |\Delta_{i, j}|\\
\leq&|S_0 * (x_{0, p} - y_{0, q}) + S_1 * (x_{1, p} - y_{1, q})| + \max\limits_{0\leq i\leq M, 0\leq j\leq N} |\Delta_{i, j}|,
\end{align*}
where $\Delta_{i, j} = \sum\nolimits_{m=0}^{1} k_{m, 0}*(\delta_{m, j}^{y}-\delta_{m, i}^{x}) + k_{m, 2}*(\delta_{m, i+1}^{x} - \delta_{m, j+1}^{y})$. Based on Cauchy–Schwarz inequality, we can get:
\begin{align*}
& |S_0 * (x_{0, p} - y_{0, q}) + S_1 * (x_{1, p} - y_{1, q})|\\
\leq& \sqrt{S_0^2 + S_1^2} * \sqrt{(x_{0, p} - y_{0, q})^2 + (x_{1, p} - y_{1, q})^2}\\
=&\sqrt{S_0^2 + S_1^2} * d_{xy}.
\end{align*}
In addition, there is always such a coupling $L^{\#}=(x_0^c, y_0^c), ..., (x_M^c, y_N^c)$ between $C(X)$ and $C(Y)$, thus:
\begin{align*}
d_F(C(X), C(Y)) \geq max(d(x_0^c, y_0^c), d(x_M^c, y_N^c)).
\end{align*}
The proof can be completed by rearranging the above formula.
\end{proof}
To verify the effectiveness of \autoref{one-conv}, we randomly selected 5000 pairs of trajectories from the Porto dataset for testing and randomly initialized a convolution kernel $k \in \mathbb{R}^{2 \times 3}$ using PyTorch. The results in ~\autoref{fig:bound}a show that the DFD between the trajectories after 1D convolution can accurately fall between the bounds of our predictions by \autoref{one-conv}. In addition, there is a positive correlation between the true DFD and the DFD after 1D convolution in ~\autoref{fig:bound}b, which implies that \textit{even for a randomly initialized kernel, the DFD can still be maintained to a large extent}.
\begin{figure}[h]
  \centering
  \includegraphics[width=\linewidth]{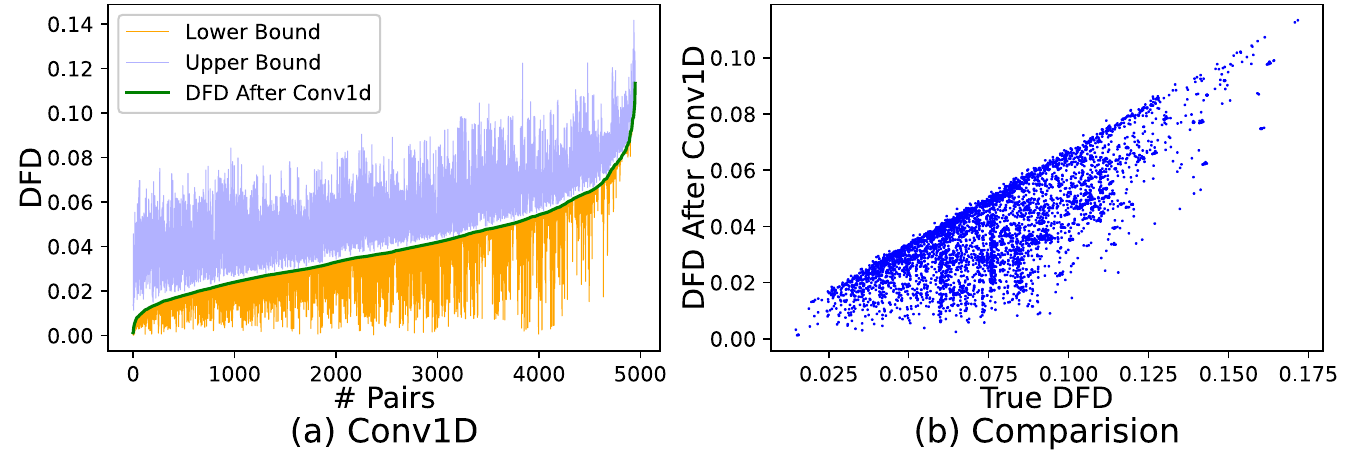}
  \caption{1D convolution bound visualization on Porto.}
  \label{fig:bound}
\end{figure}
\begin{theorem}[One-dimensional Max-Pooling Bound]
\label{one-pool}
Given two sequences $X = [x_1, ..., x_M]$, $Y = [y_1, ..., y_N]$, and each $x_i\in X(y_i\in Y)$ is a $l$-dimensional vector, i.e., $x_i=[x_{i, 1}, ..., x_{i, l}]^{\mathsf{T}}$. A one-dimensional max pooling operation $P(\cdot)$ on $X, Y$ with size $k$ and stride $k$, assuming that $M$ and $N$ are divisible by $k$. Then the following holds:
\[ d_F(X, Y) - bound \leq d_F(P(X), P(Y)) \leq d_F(X, Y) + bound, \]
in which
\[ bound = max \{ d(X_{i}^{\downarrow}, X_{i}^{\uparrow})| 1\leq i\leq\frac{M}{k}\} + max \{ d(Y_{i}^{\downarrow}, Y_{i}^{\uparrow})| 1\leq i\leq\frac{N}{k}\}, \]
and $X_{i}^{\downarrow} = [x_{i, 1}^\downarrow, ..., x_{i, l}^\downarrow]^{\mathsf{T}}$, $x_{i, j}^\downarrow = min \{ x_{t, j}|t\in[(i-1)*k+1, i*k)\}$; $X_{i}^{\uparrow} = [x_{i, 1}^\uparrow, ..., x_{i, l}^\uparrow]^{\mathsf{T}}$, $x_{i, j}^\uparrow = max \{ x_{t, j}|t\in[(i-1)*k+1, i*k)\}$.(The same goes for $Y_{i}^{\downarrow}$ and $Y_{i}^{\uparrow}$)
\end{theorem}

\begin{proof}
Based on the triangle inequality of DFD, we can get:
\begin{align*}
d_F(X, Y) & \leq d_F(X, P(X)) + d_F(P(X), Y) \\
& \leq d_F(X, P(X)) + d_F(P(X), P(Y)) + d_F(P(Y), Y).
\end{align*}
Using this property again, we have:
\[ d_F(P(X), P(Y)) \leq d_F(X, Y) + (d_F(P(X), X) + d_F(Y, P(Y))). \]
Rearrange these two inequalities, we can get:
\begin{align*}    
    bound = d_F(X, P(X)) + d_F(Y, P(Y)).
\end{align*}
Suppose $P(X) = [x_1^p, ..., x_{\frac{M}{k}}^p]$, and each $x_i^p$ is a $l$-dimensional vector, i.e., $x_i^p=[x_{i, 1}^p, ..., x_{i, l}^p]^{\mathsf{T}}$, where $x_{i, j}^p = max \{ x_{t, j}|t\in[(i-1)*k+1, i*k)\}$.
Then for $d_F(X, P(X))$, we can always construct such a coupling $L^*=\underbrace{(x_1, x_1^p), ..., (x_k, x_1^p)}_{\mbox{k}}, ..., \underbrace{(x_{M-k+1}, x_{\frac{M}{k}}^p), ..., (x_{M}, x_{\frac{M}{k}}^p)}_{\mbox{k}}$.
Thus $d_F(X, P(X)) \leq max_{(x_i, x_j^p) \in L^*} d(x_i, x_j^p)$.
In this way, we divide the coupling $L^*$ into $\frac{M}{k}$ groups. Without loss of generality, we take out the $t$-th group, that is:
\[ (x_{(t - 1)*k + 1}, x_t^p), ..., (x_{t * k}, x_t^p), \]
thus for $i\in[(t-1)*k+1, t*k)$, we have:
\begin{align*}
max(d(x_{i}, x_t^p))=&max(d([x_{i, 1}, ..., x_{i, l}]^{\mathsf{T}}, [x_{t, 1}^p, ..., x_{t, l}^p]^{\mathsf{T}}))\\
=&max(d([x_{i, 1}, ..., x_{i, l}]^{\mathsf{T}}, X_t^{\uparrow}))\\
\leq&d(X_t^{\downarrow}, X_t^{\uparrow}).
\end{align*}
Using this bound to $d_F(Y, P(Y))$ completes the proof.
\end{proof}

To verify the effectiveness of \autoref{one-pool}, we use the same setting as ~\autoref{one-conv} for testing. The size and stride of max-pooling are set to 2, i.e., $k=2$. As shown in~\autoref{fig:max-bound}a, the DFD between the trajectories after max-pooling can also accurately fall between the bounds of our predictions by \autoref{one-pool}. In addition, ~\autoref{fig:max-bound}b shows that the real DFD value has almost no change compared with the DFD after max-pooling, this implies that \textit{max-pooling is a suitable technique that can reduce the dimensionality of trajectory sequences with almost no loss of effective features that are important for DFD-based similarity recognition}.
\begin{figure}[h]
  \centering
  \includegraphics[width=\linewidth]{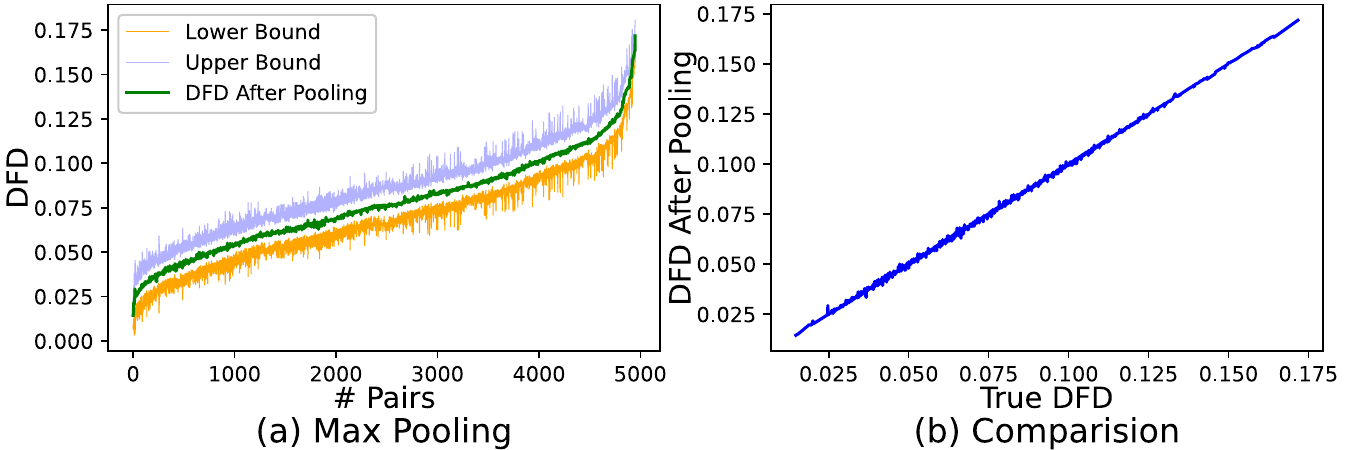}
  \caption{1D max-pooling bound visualization on Porto.}
  \label{fig:max-bound}
\end{figure}
\subsection{Two-dimensional Convolution}
\begin{definition}[Trajectory MBR Distance]
Given two trajectories $X$ and $Y$, we denote the Minimum Bounding Rectangles (MBRs) of $X$ and $Y$ as $X_{mbr}$ and $Y_{mbr}$ based on the minimum and maximum longitude and latitude of the trajectory. We thus define the distance between $X_{mbr}$ and $Y_{mbr}$ is:
\[ dist(X_{mbr}, Y_{mbr}) = \min\limits_{p\in X_{mbr}, q \in Y_{mbr}} d(p, q), \]
where $d(\cdot,\cdot)$ is Euclidean distance between two points.
\end{definition}

\begin{theorem}[Two-dimensional Convolution Bound]
\label{two-conv}
Given two trajectories $X$ and $Y$, their MBRs are $X_{mbr}$ and $Y_{mbr}$.
We denote the binary images of $X$ and $Y$ based on the grid width $\delta$ as $X_{BI}$, $Y_{BI}$, and the non-0 pixels contained in $X_{BI}$ and $Y_{BI}$ are $n$ and $m$ respectively.
A two-dimensional convolution operation $C(\cdot)$ on $X_{BI}$ and $Y_{BI}$ with stride 1, padding 1, kernel $k \in \mathbb{R+}^{3 \times 3}$. If $dist(X_{mbr}, Y_{mbr}) > 2\sqrt{2}*\delta$, then we have $d_{F}(X, Y) > 2\sqrt{2}*\delta$, and
\[d(C(X_{BI}), C(Y_{BI})) \geq \sqrt{(n+m)*\sum_{i=0}^{2}\sum_{j=0}^{2}(k_{i,j})^2},\]
where $d(\cdot,\cdot)$ is Euclidean distance between two vectors.
\end{theorem}

\begin{proof}
We can easily get $\min\limits_{p\in X, q \in Y} d(p, q) > 2\sqrt{2}*\delta$ based on $dist(X_{mbr}, Y_{mbr})> 2\sqrt{2}*\delta$, thus:
\begin{align*}
d_F(X, Y) &= \min\limits_{L} \{ \max\limits_{(p_i, q_j)\in L} d(p_i, q_j)\} \geq \min \limits_{p_i\in X, q_j \in Y} d(p_i, q_j)\\
&=\min\limits_{p\in X, q \in Y} d(p, q)> 2\sqrt{2}*\delta.
\end{align*}
Due to $dist(X_{mbr}, Y_{mbr})>2\sqrt{2}*\delta$, there must be:
\[dist(X_{BI_{mbr}}, Y_{BI_{mbr}})\geq 2.\]
Since stride and padding are 1, we can easily deduce that there is no overlap between $C(X_{BI})$ and $C(Y_{BI})$, thus:
\[d(C(X_{BI}), C(Y_{BI})) = \sqrt{\left\|C(X_{BI})\right\|_{2}^2 + \left\|C(Y_{BI})\right\|_{2}^2}.\]
Considering that $C(X_{BI})$ is essentially a superposition of $n$ kernels at different locations, and $k_{i, j} > 0$, thus in any case there is:
\[\left\|C(X_{BI})\right\|_{2}^2 \geq  n * \sum_{i=0}^{2}\sum_{j=0}^{2}(k_{i,j})^2 \]
Applying this bound to $\left\|C(Y_{BI})\right\|_{2}^2$ completes the proof.
\end{proof}

\section{Experiments}
\subsection{Experimental Setting}
\textbf{Datasets.}
We evaluate the performance of ConvTraj using four widely used real-world datasets: \textbf{Geolife\footnote{https://www.microsoft.com/en-us/research/publication/geolife-gps-trajectory-dataset-user-guide/}}, \textbf{Porto\footnote{https://www.kaggle.com/competitions/pkdd-15-predict-taxi-service-trajectory-i/data}}, \textbf{Chengdu} and \textbf{TrajCL-Porto\footnote{https://github.com/changyanchuan/TrajCL}}. For Geolife and Porto, we preprocess them using the method in~\cite{DBLP:conf/icde/YaoCZB19}, i.e. selecting trajectories in the central area of the city and removing items with less than 10 records. For TrajCL-Porto, it's an open-source dataset of TrajCL\cite{DBLP:conf/icde/Chang0LT23}, we thus do not perform any processing. For Chengdu, we randomly selected 5000 trajectories from this dataset. The properties of these datasets are shown in \autoref{tab:datasset}.

\begin{table}[h]
  \centering
  \caption{Trajectory Dataset Properties}
  \label{tab:datasset}
  \resizebox{\columnwidth}{!}{%
  \begin{tabular}{ccccc}
    \toprule
    \textbf{Dataset}  & \textbf{Geolife} & \textbf{Porto} & \textbf{Chengdu} & \textbf{TrajCL-Porto}\\ \midrule
    \# Total Items                    & $13386$        & $1601579$ & $5000$ & $9000$  \\
    \# Training Items                    & $3000$        & $3000$ & $1000$ & $7000$ \\
    \# Query Items                    & $1000$        & $500$ & $1000$& $2000$ \\
    \# Candidate Items                    & $9386$        & $1598079$ & $3000$& $2000$  \\
    \hline
    Avg-(\# Points)         & $437.80$          & $48.91$ & $228.44$  & $49.72$\\
    Min-(\# Points)         & $11$        & $11$ & $29$    & $20$\\
    Max-(\# Points)         & $7579$        & $3836$ & $1575$ & $200$\\
    Lat-Lon Area                      &\makecell{(116.20, 116.50) \\ (39.85, 40.07)}&\makecell{(-8.73, -8.50)\\(41.10, 41.24)} &\makecell{(104.04, 114.10) \\ (30.65, 30.73)}&\makecell{(-8.70, -8.52)\\(41.10, 41.208)}\\
    \bottomrule
\end{tabular}
}
\end{table}

\begin{table*}[]
\caption{Embedding Results On Geolife dataset (3 runs)}
\label{tab:metric-result-haus-dfd}
\resizebox{0.95\textwidth}{!}{%
\begin{tabular}{c|ccccc|ccccc}
\toprule
\multicolumn{11}{c}{\textbf{Geolife}}                     \\
\hline
&\multicolumn{5}{c|}{\textbf{Hausdorff}}                   & \multicolumn{5}{c}{\textbf{DFD}}                     \\ 
                    Model   & HR@1 & HR@5 & HR@10 & HR@50 & R10@50 & HR@1 & HR@5 & HR@10 & HR@50 & R10@50 \\ \midrule
t2vec                  &   $22.00_{\pm 0.38} $   &   $22.82_{\pm 0.57} $   &   $24.48_{\pm 0.42} $      &   $26.64_{\pm 0.14} $    &   $44.60_{\pm 0.11} $          &   $25.70_{\pm 0.53} $   &   $26.36_{\pm 0.22} $   &   $28.33_{\pm 0.43} $  &   $32.44_{\pm 0.21} $    &   $53.45_{\pm 0.14} $         \\    
TrjSR                  &   $28.30_{\pm 0.23} $   &   $34.86_{\pm 0.12} $   &   $37.26_{\pm 0.11} $      &   $42.75_{\pm 0.06} $    &   $66.56_{\pm 0.01} $          &   $25.70_{\pm 0.32} $   &   $29.92_{\pm 0.31} $   &   $33.29_{\pm 0.09} $  &   $36.94_{\pm 0.04} $    &   $61.62_{\pm 0.01} $         \\  
TrajCL                 &   $22.03_{\pm 0.49} $   &   $31.08_{\pm 0.02} $   &   $37.21_{\pm 0.01} $      &   $52.49_{\pm 0.07} $    &   $72.55_{\pm 0.08} $          &   $25.40_{\pm 0.51} $   &   $33.51_{\pm 0.12} $   &   $38.98_{\pm 0.32} $  &   $54.72_{\pm 0.39} $    &   $75.82_{\pm 0.43} $         \\  
NeuTraj              &   $\underline{34.53_{\pm 1.51}} $   &   $\underline{42.31_{\pm 0.18}} $   &   $\underline{48.40_{\pm 0.18}} $      &   $\underline{61.88_{\pm 0.14}} $    &   $\underline{80.38_{\pm 0.01}} $          &   $\underline{46.37_{\pm 1.40}} $   &   $\underline{58.50_{\pm 0.12}} $   &   $\underline{64.47_{\pm 0.42}} $  &   $\underline{76.43_{\pm 0.47}} $    &   $\underline{94.44_{\pm 0.41}} $         \\  
Traj2SimVec               &   $26.46_{\pm 1.34} $   &   $33.49_{\pm 0.96} $   &   $42.39_{\pm 0.36} $      &   $48.26_{\pm 0.31} $    &   $65.12_{\pm 0.23} $          &   $27.13_{\pm 0.94} $   &   $40.39_{\pm 0.98} $   &   $42.75_{\pm 0.33} $  &   $50.27_{\pm 0.26} $    &   $70.20_{\pm 0.02} $         \\  
TrajGAT               &   $19.80_{\pm 2.44} $   &   $25.80_{\pm 1.23} $   &   $30.57_{\pm 0.74} $      &   $44.03_{\pm 0.23} $    &   $66.69_{\pm 0.24} $          &   $17.30_{\pm 2.12} $   &   $21.20_{\pm 0.98} $   &   $25.87_{\pm 1.20} $  &   $37.90_{\pm 0.23} $    &   $61.91_{\pm 0.98} $         \\  
\hline
\textbf{ConvTraj}   &   $\textbf{46.17}_{\pm 2.26} $    &   $\textbf{57.73}_{\pm 0.22} $     &   $\textbf{63.69}_{\pm 0.37} $      &   $\textbf{76.12}_{\pm 0.03} $    &   $\textbf{95.20}_{\pm 0.00} $         &   $\textbf{51.80}_{\pm 0.93} $    &   $\textbf{62.73}_{\pm 0.56} $  &   $\textbf{68.86}_{\pm 0.33} $ &   $\textbf{79.52}_{\pm 0.12} $    &   $\textbf{97.34}_{\pm 0.01} $     \\ 
Gap with SOTA  &   $+11.64 $   &  $+15.42 $    & $+15.29 $      &   $+14.24 $    &   $+14.82 $          &   $+5.43 $   &   $+4.23 $   &     $+4.39 $  &   $+3.09 $    &  $+2.90 $       \\
\bottomrule
&\multicolumn{5}{c|}{\textbf{DTW}}                   & \multicolumn{5}{c}{\textbf{EDR}}                     \\
\hline
t2vec                  &   $25.30_{\pm 0.32} $   &   $26.70_{\pm 0.77} $   &   $28.91_{\pm 0.43} $      &   $32.81_{\pm 0.25} $    &   $55.40_{\pm 0.21} $          &   $18.70_{\pm 0.97} $   &   $18.34_{\pm 0.24} $   &   $20.95_{\pm 0.45} $  &   $\underline{25.56_{\pm 0.22}} $    &   $\underline{46.22_{\pm 0.10}} $         \\      
TrjSR                  &   $\underline{27.80_{\pm 0.13}} $   &   $\underline{32.64_{\pm 0.21}} $   &   $\underline{36.52_{\pm 0.09}} $      &   $40.74_{\pm 0.03} $    &   $\underline{67.91_{\pm 0.04}} $          &   $\underline{20.50_{\pm 0.34}} $   &   $18.26_{\pm 0.42} $   &   $19.83_{\pm 0.19} $  &   $23.74_{\pm 0.08} $    &   $45.52_{\pm 0.03} $         \\    
TrajCL                 &   $8.47_{\pm 0.38} $   &   $12.48_{\pm 0.14} $   &   $14.63_{\pm 0.02} $      &   $19.16_{\pm 0.05} $    &   $32.05_{\pm 0.08} $          &   $17.00_{\pm 0.89} $   &   $\underline{20.33_{\pm 0.15}} $   &   $\underline{22.74_{\pm 0.05}} $  &   $24.67_{\pm 0.18} $    &   $44.90_{\pm 0.34} $         \\  
NeuTraj              &   $25.13_{\pm 0.31} $   &   $30.57_{\pm 0.09} $   &   $33.68_{\pm 0.22} $      &   $\underline{41.72_{\pm 0.17}} $    &   $62.87_{\pm 0.56} $          &   $8.70_{\pm 16.81} $   &   $11.98_{\pm 16.48} $   &   $14.28_{\pm 19.85} $  &   $19.59_{\pm 14.78} $    &   $21.63_{\pm 10.99} $         \\ 
Traj2SimVec               &   $16.22_{\pm 0.84} $   &   $13.24_{\pm 0.48} $   &   $15.39_{\pm 0.75} $      &   $20.37_{\pm 0.09} $    &   $35.09_{\pm 0.05} $          &   $7.90_{\pm 0.95} $   &   $10.44_{\pm 0.64} $   &   $12.39_{\pm 0.70} $  &   $16.02_{\pm 0.11} $    &   $18.45_{\pm 0.01} $         \\  
TrajGAT               &   $13.80_{\pm 0.45} $   &   $19.98_{\pm 0.23} $   &   $24.77_{\pm 0.47} $      &   $35.26_{\pm 0.98} $    &   $59.36_{\pm 0.23} $          &   $13.40_{\pm 0.32} $   &   $15.58_{\pm 0.78} $   &   $17.84_{\pm 1.32} $  &   $22.49_{\pm 0.44} $    &   $36.78_{\pm 0.24} $         \\   
\hline
\textbf{ConvTraj}   &   $\textbf{31.70}_{\pm 0.72} $    &   $\textbf{41.56}_{\pm 0.38} $     &   $\textbf{46.46}_{\pm 0.67} $      &   $\textbf{59.26}_{\pm 0.44} $    &   $\textbf{83.70}_{\pm 0.02} $         &   $\textbf{25.96}_{\pm 1.87} $    &   $\textbf{26.95}_{\pm 2.53} $  &   $\textbf{28.64}_{\pm 2.93} $ &   $\textbf{30.75}_{\pm 1.40} $    &   $\textbf{54.93}_{\pm 4.93} $     \\
Gap with SOTA  &   $+3.90 $   &  $+8.92 $    & $+9.94 $      &   $+17.54 $    &   $+15.79 $          &   $+5.46 $   &   $+6.62 $   &     $+5.90 $  &   $+5.19 $    &  $+8.71 $       \\

\bottomrule
\end{tabular}%
}
\end{table*}

\begin{table*}[]
\caption{Embedding Results On Porto dataset (3 runs)}
\label{tab:metric-result-dtw-edr}
\resizebox{0.95\textwidth}{!}{%
\begin{tabular}{c|ccccc|ccccc}
\toprule
\multicolumn{11}{c}{\textbf{Porto}}                     \\
\hline
&\multicolumn{5}{c|}{\textbf{Hausdorff}}                   & \multicolumn{5}{c}{\textbf{DFD}}                     \\
                    Model   & HR@1 & HR@5 & HR@10 & HR@50 & R10@50 & HR@1 & HR@5 & HR@10 & HR@50 & R10@50 \\ \midrule
t2vec                  &   $4.00_{\pm 1.01} $   &   $5.88_{\pm 0.87} $   &   $7.28_{\pm 0.89} $      &   $10.46_{\pm 0.13} $    &   $17.08_{\pm 0.04} $          &   $5.20_{\pm 0.99} $   &   $6.28_{\pm 0.98} $   &   $7.66_{\pm 0.82} $  &   $11.09_{\pm 0.09} $    &   $17.84_{\pm 0.01} $         \\       
TrjSR                  &   $6.87_{\pm 0.12} $   &   $13.26_{\pm 0.74} $   &   $14.79_{\pm 0.23} $      &   $26.71_{\pm 0.25} $    &   $33.46_{\pm 0.02} $          &   $8.12_{\pm 0.30} $   &   $10.36_{\pm 0.32} $   &   $14.26_{\pm 0.64} $  &   $20.37_{\pm 0.03} $    &   $37.48_{\pm 0.11} $         \\    
TrajCL                 &   $7.07_{\pm 0.81} $   &   $12.55_{\pm 0.18} $   &   $15.37_{\pm 0.12} $      &   $23.47_{\pm 0.02} $    &   $35.37_{\pm 0.06} $          &   $7.07_{\pm 0.33} $   &   $14.08_{\pm 0.18} $   &   $18.01_{\pm 0.04} $  &   $28.31_{\pm 0.11} $    &   $42.97_{\pm 0.15} $         \\ 
NeuTraj              &   $\underline{9.30_{\pm 0.09}} $   &   $\underline{19.32_{\pm 0.01}} $   &   $\underline{24.07_{\pm 0.02}} $      &   $\underline{34.22_{\pm 0.00}} $    &   $\underline{51.43_{\pm 0.07}} $          &   $\underline{15.13_{\pm 0.44}} $   &   $\underline{27.71_{\pm 0.14}} $   &   $\underline{33.64_{\pm 0.09}} $  &   $\underline{45.47_{\pm 0.01}} $    &   $\underline{66.58_{\pm 0.14}} $         \\  
Traj2SimVec        &   $6.34_{\pm 0.84} $   &   $14.33_{\pm 0.95} $   &   $16.32_{\pm 0.27} $      &   $27.34_{\pm 0.22} $    &   $37.45_{\pm 0.02} $          &   $7.64_{\pm 0.34} $   &   $16.37_{\pm 0.32} $   &   $20.03_{\pm 0.06} $  &   $30.09_{\pm 0.02} $    &   $44.11_{\pm 0.22} $         \\  
TrajGAT               &   $6.48_{\pm 0.96} $   &   $16.48_{\pm 0.85} $   &   $18.29_{\pm 0.61} $      &   $22.31_{\pm 0.23} $    &   $43.58_{\pm 0.31} $          &   $8.39_{\pm 1.23} $   &   $13.49_{\pm 0.84} $   &   $20.38_{\pm 0.43} $  &   $24.31_{\pm 0.54} $    &   $39.78_{\pm 0.21} $         \\  
\hline
\textbf{ConvTraj}   &   $\textbf{15.33}_{\pm 2.11} $    &   $\textbf{27.68}_{\pm 0.09} $     &   $\textbf{33.27}_{\pm 0.13} $      &   $\textbf{45.98}_{\pm 0.11} $    &   $\textbf{67.20}_{\pm 0.04} $         &   $\textbf{22.73}_{\pm 0.46} $    &   $\textbf{34.97}_{\pm 0.21} $  &   $\textbf{40.59}_{\pm 0.06} $ &   $\textbf{53.35}_{\pm 0.18} $    &   $\textbf{77.33}_{\pm 0.59} $     \\ 
Gap with SOTA  &   $+6.03 $   &  $+8.36 $    & $+9.20 $      &   $+11.76 $    &   $+15.77 $          &   $+7.30 $   &   $+7.26 $   &     $+6.95 $  &   $+7.88 $    &  $+10.75 $       \\
\bottomrule
&\multicolumn{5}{c|}{\textbf{DTW}}                   & \multicolumn{5}{c}{\textbf{EDR}}                     \\
\hline
t2vec                  &   $6.80_{\pm 0.98} $   &   $8.16_{\pm 0.79} $   &   $9.60_{\pm 0.69} $      &   $12.53_{\pm 0.44} $    &   $20.92_{\pm 0.13} $          &   $3.00_{\pm 0.88} $   &   $4.24_{\pm 0.91} $   &   $4.76_{\pm 0.67} $  &   $7.77_{\pm 0.18} $    &   $12.30_{\pm 0.23} $         \\      
TrjSR                  &   $\underline{9.66_{\pm 1.05}} $   &   $\underline{15.44_{\pm 0.31}} $   &   $\underline{18.29_{\pm 0.48}} $      &   $21.08_{\pm 0.28} $    &   $\underline{38.23_{\pm 0.15}} $          &   $4.76_{\pm 0.63} $   &   $6.13_{\pm 0.17} $   &   $8.42_{\pm 0.13} $  &   $14.87_{\pm 0.23} $    &   $23.20_{\pm 0.55} $         \\  
TrajCL                 &   $0.53_{\pm 0.14} $   &   $1.51_{\pm 0.06} $   &   $2.33_{\pm 0.05} $      &   $4.56_{\pm 0.18} $    &   $6.96_{\pm 0.99} $          &   $\underline{5.73_{\pm 0.44}} $   &   $\underline{7.99_{\pm 0.01}} $   &   $\underline{10.47_{\pm 0.12}} $  &   $\underline{16.52_{\pm 0.12}} $    &   $\underline{25.53_{\pm 0.58}} $         \\  
NeuTraj              &   $8.40_{\pm 0.19} $   &   $13.48_{\pm 0.03} $   &   $16.33_{\pm 0.07} $      &   $\underline{22.98_{\pm 0.02}} $    &   $34.81_{\pm 0.17} $          &   $1.50_{\pm 0.81} $   &   $2.52_{\pm 0.19} $   &   $3.79_{\pm 0.12} $  &   $7.78_{\pm 0.06} $    &   $10.83_{\pm 0.15} $         \\  
Traj2SimVec               &   $7.26_{\pm 0.26} $   &   $10.44_{\pm 0.32} $   &   $13.22_{\pm 0.37} $      &   $15.30_{\pm 0.26} $    &   $18.30_{\pm 0.09} $          &   $1.20_{\pm 0.23} $   &   $1.59_{\pm 0.31} $   &   $1.85_{\pm 0.13} $  &   $4.96_{\pm 0.15} $    &   $7.33_{\pm 0.29} $         \\  
TrajGAT               &   $5.73_{\pm 0.11} $   &   $11.02_{\pm 0.23} $   &   $12.58_{\pm 0.08} $      &   $16.97_{\pm 0.07} $    &   $20.79_{\pm 0.01} $          &   $3.22_{\pm 0.23} $   &   $4.78_{\pm 0.67} $   &   $6.23_{\pm 0.02} $  &   $10.98_{\pm 0.04} $    &   $12.34_{\pm 0.01} $         \\    
\hline
\textbf{ConvTraj}   &   $\textbf{10.13}_{\pm 0.33} $    &   $\textbf{19.77}_{\pm 0.18} $     &   $\textbf{24.46}_{\pm 0.41} $      &   $\textbf{35.36}_{\pm 0.15} $    &   $\textbf{52.83}_{\pm 0.03} $         &   $\textbf{13.47}_{\pm 0.70} $    &   $\textbf{13.84}_{\pm 0.58} $  &   $\textbf{15.61}_{\pm 0.34} $ &   $\textbf{21.45}_{\pm 0.78} $    &   $\textbf{37.00}_{\pm 2.80} $     \\  
Gap with SOTA  &   $+0.47 $   &  $+4.33 $    & $+6.17 $      &   $+12.38 $    &   $+14.60 $          &   $+7.74 $   &   $+5.85 $   &     $+5.14 $  &   $+4.93 $    &  $+11.47 $       \\

\bottomrule
\end{tabular}%
}
\end{table*}

\begin{table*}[]
\caption{Embedding Results On Chengdu dataset}
\label{tab:metric-result-haus-dfd-chengdu}
\resizebox{0.95\textwidth}{!}{%
\begin{tabular}{c|ccc|ccc|ccc|ccc}
\toprule
\multicolumn{13}{c}{\textbf{Chengdu}}                     \\
\hline
&\multicolumn{3}{c|}{\textbf{Hausdorff}}                   & \multicolumn{3}{c|}{\textbf{DFD}}    &\multicolumn{3}{c|}{\textbf{DTW}}&\multicolumn{3}{c}{\textbf{EDR}}                 \\ 
                    Model   & HR@5 & HR@10 & R10@50 & HR@5 & HR@10 & R10@50 & HR@5 & HR@10 & R10@50 & HR@5 & HR@10 & R10@50\\ \midrule
t2vec                    &   $ 16.16$   &   $16.52$       &   $30.69$           &   $21.34$   &   $22.34$     &   $44.06$         &   $21.00$   &   $ 22.68$          &   $46.03$            &   $14.18$   &   $16.37$     &   $37.41$   \\   
TrjSR                     &   $10.44$   &   $12.09$         &   $25.44$           &   $11.08$   &   $12.77$     &   $25.96$         &   $16.42$   &   $18.32$          &   $36.63$            &   $\underline{18.80}$   &   $19.78$      &   $40.02$ \\ 
TrajCL                 &   $\underline{22.14} $   &   $\underline{27.04} $         &   $\underline{61.93} $           &   $19.00 $   &   $21.42 $      &   $51.79 $         &   $23.48 $   &   $27.20 $        &   $59.89 $             &   $17.40$   &   $20.53$ &   $\underline{48.58} $         \\
NeuTraj               &   $21.82$   &   $23.27$         &   $39.09$             &   $\underline{32.28}$   &   $\underline{34.70}$      &   $\underline{49.51}$     &   $\underline{28.40}$   &   $\underline{29.33}$          &   $46.90$             &   $10.44$   &   $11.54$      &   $24.65$    \\     
Traj2SimVec                 &   $18.34$   &   $20.17$        &   $43.67$             &   $25.16$   &   $28.33$   &   $42.78$      &   $18.66$   &   $19.37$          &   $40.69$            &   $ 6.93$   &   $ 8.31$     &   $19.18$\\
TrajGAT                 &   $19.98$   &   $25.26$        &   $57.57$             &   $16.24$   &   $19.11$   &   $49.18$      &   $23.96$   &   $28.32$          &   $\underline{65.41}$            &   $ 17.94$   &   $ \underline{20.55}$     &   $48.40$\\  
\hline
\textbf{ConvTraj}       &   $\textbf{36.26} $     &   $\textbf{42.78} $         &   $\textbf{76.67} $            &   $\textbf{53.34} $  &   $\textbf{58.18} $     &   $\textbf{93.14} $     &   $\textbf{34.90} $     &   $\textbf{40.40} $          &   $\textbf{76.23} $             &   $\textbf{21.50} $  &   $\textbf{25.34} $     &   $\textbf{55.21} $     \\
Gap with SOTA     &  $+14.12 $    & $+15.74 $         &   $+19.10 $           &   $+21.06 $   &     $+23.48 $      &  $+41.45 $       &  $+6.50 $    & $+11.07 $         &   $+10.82 $             &   $+2.70 $   &     $+4.81 $      &  $+6.63 $       \\
\bottomrule
\end{tabular}%
}
\end{table*}

\begin{table*}[h]
\caption{Embedding Results On TrajCL-Porto dataset}
\label{tab:trajcl-result-haus-dfd}
\resizebox{0.95\textwidth}{!}{%
\begin{tabular}{c|ccc|ccc|ccc|ccc}
\toprule
\multicolumn{13}{c}{\textbf{TrajCL-Porto}}                     \\
\hline
&\multicolumn{3}{c|}{\textbf{EDR}}                   & \multicolumn{3}{c}{\textbf{EDwP}} &\multicolumn{3}{c|}{\textbf{Hausdorff}}                   & \multicolumn{3}{c}{\textbf{DFD}}                     \\
                    Model   &  HR@5 & HR@20 & R5@20 & HR@5 & HR@20 & R5@20&  HR@5 & HR@20 & R5@20 & HR@5 & HR@20 & R5@20 \\ \midrule
t2vec                  &   $0.125 $   &   $0.164 $   &   $0.286 $      &   $0.399 $    &   $0.518 $          &   $0.751 $&   $0.405 $   &   $0.549 $   &   $0.770 $      &   $0.504 $    &   $0.651 $          &   $0.883 $   \\    
TrjSR                  &   $0.137 $   &   $0.147 $   &   $0.273 $      &   $0.271 $    &   $0.346 $          &   $0.535 $&   $0.541 $   &   $0.638 $   &   $0.880 $      &   $0.271 $    &   $0.356 $          &   $0.523 $   \\    
E2DTC~\cite{DBLP:conf/icde/FangDCHGC21}                  &   $0.122 $   &   $0.157 $   &   $0.272 $      &   $0.390 $    &   $0.514 $          &   $0.742 $&   $0.391 $   &   $0.537 $   &   $0.753 $      &   $0.498 $    &   $0.648 $          &   $0.879 $   \\    
CSTRM~\cite{DBLP:journals/comcom/LiuTGCZ22}                  &   $0.138 $   &   $0.191 $   &   $0.321 $      &   $0.415 $    &   $0.536 $          &   $0.753 $&   $0.459 $   &   $0.584 $   &   $0.813 $      &   $0.421 $    &   $0.557 $          &   $0.768 $   \\    
TrajCL                &   \underline{$0.172 $ }  &   $\textbf{0.222 }$   &  \underline{$ 0.376    $}   &  \underline{ $0.546 $   } &   \underline{$0.646 $ }         &   \underline{$0.881 $}&   $0.643 $   &   $0.721 $   &   $0.954 $      &   \underline{$0.618 $}    &   \underline{$0.740 $}          &   \underline{$0.955 $}   \\    
T3S~\cite{DBLP:conf/icde/YangW0Q0021}                    &   $0.140 $   &   \underline{$0.192 $}   &   $0.325 $      &   $0.377 $    &   $0.498 $          &   $0.702 $&   $0.329 $   &   $0.482 $   &   $0.668 $      &   $0.595 $    &   $0.728 $          &   $0.946 $   \\    
Traj2SimVec            &   $0.119 $   &   $0.163 $   &   $0.285 $      &   $0.172 $    &   $0.253 $          &   $0.390 $&   $0.339 $   &   $0.429 $   &   $0.543 $      &   $0.529 $    &   $0.664 $          &   $0.894 $   \\    
TrajGAT                &   $0.090 $   &   $0.102 $   &   $0.184 $      &   $0.201 $    &   $0.274 $          &   $0.469 $&   \underline{$0.686 $}   &   \underline{$0.740 $}   &   \underline{$0.969 $}      &   $0.362 $    &   $0.403 $          &   $0.704 $   \\    
\hline
\textbf{ConvTraj}      & $\textbf{0.292}$   &   $0.181 $   &   $\textbf{0.414} $      &   $\textbf{0.776 }$    &   $\textbf{0.826   }$        &  $\textbf{ 0.987 } $&   $\textbf{0.754}$   &   $\textbf{0.770 }$   &   $\textbf{0.983 }$      &   $\textbf{0.760 }$    &   $\textbf{0.786 }$          &   $\textbf{0.984 }$   \\    
Gap with SOTA   &   $+0.12 $   &   $-0.041 $   &   $+0.038 $      &   $+0.23 $    &   $+0.18 $          &   $+0.106 $&   $+0.068 $   &   $+0.03 $   &   $+0.014 $      &   $+0.142 $    &   $+0.046 $          &   $+0.029 $   \\    
\bottomrule
ConvTraj-1D CNN                  &   $0.230 $   &   $0.097 $   &   $0.279 $      &   $0.648 $    &   $0.685 $          &   $0.937 $&   $0.732 $   &   $0.757 $   &   $0.983 $      &   $0.736 $    &   $0.769 $          &   $0.978 $   \\
ConvTraj-2D CNN                  &   $0.285 $   &   $0.174 $   &   $0.387 $      &   $0.611 $    &   $0.586 $          &   $0.949 $&   $0.746 $   &   $0.769 $   &   $0.983 $      &   $0.565 $    &   $0.528 $          &   $0.908 $   \\    
\bottomrule
\end{tabular}%
}
\end{table*}

\textbf{Baselines.}
When we test on Geolife, Porto, and Chengdu, we follow existing works~\cite{DBLP:conf/kdd/YaoHDCHB22,DBLP:journals/corr/abs-2311-00960} and compare ConvTraj with six representative methods, including \textbf{t2vec~\cite{DBLP:conf/icde/LiZCJW18}} and \textbf{TrjSR~\cite{DBLP:conf/ijcnn/CaoTWWX21}} based on self-supervised learning; \textbf{NeuTraj~\cite{DBLP:conf/icde/YaoCZB19}}, \textbf{Traj2SimVec~\cite{DBLP:conf/ijcai/ZhangZJZSSW20}}, \textbf{TrajGAT~\cite{DBLP:conf/kdd/YaoHDCHB22}}, and \textbf{TrajCL~\cite{DBLP:conf/icde/Chang0LT23}} based on supervised learning. For the self-supervised method, since its goal is not to approximate the existing measurements, we thus perform the following steps to handle it. We first randomly select a part of the trajectory for pre-training (We select 10000 trajectories for Geolife and 200000 for Porto. Since TrajCL also needs pre-training, we will use these data to pre-train TrajCL). 
Then we add an MLP encoder in the end and fine-tune it with the triplet selection method and loss function in ~\autoref{sec-training-pipeline}. For those methods which have open-source code~\cite{DBLP:conf/icde/LiZCJW18,DBLP:conf/icde/Chang0LT23,DBLP:conf/icde/YaoCZB19,DBLP:conf/kdd/YaoHDCHB22, DBLP:conf/ijcnn/CaoTWWX21}, we directly use their implementation. For others ~\cite{DBLP:conf/ijcai/ZhangZJZSSW20}, we implement it based on the settings of its paper. In addition, since many baselines have been evaluated on the TrajCL-Porto and the results have been reported in \cite{DBLP:conf/icde/Chang0LT23}, we will directly compare our results with those of other baselines reported in \cite{DBLP:conf/icde/Chang0LT23}.

\textbf{Metrics.} We follow existing works~\cite{DBLP:conf/icde/YaoCZB19,DBLP:conf/ijcai/ZhangZJZSSW20, DBLP:conf/kdd/YaoHDCHB22} and evaluate the effectiveness of these methods using the task of $k$ nearest neighbor search. Specifically, we first use the top-$k$ hitting rate (HR@$k$), which is the overlap percentage of detection top-$k$ results with the ground truth. The second is the top-50 recall of the top-10 ground truth (R10@50), i.e. how many top 10 ground truths are recovered by the generated top 50 lists. The calculation of these two types of metrics is very close. For HR@$k$, we first find the top-$k$ most similar trajectories for each query in the candidate set. Then, for each query, trajectories in the candidate set are ranked according to their distance to the query in the embedding space. If the trajectories ranking top $k$ contain $k^\prime$ of the true top-$k$ neighbors, the HR@$k$ is $k^\prime/k$. 
For R10@50, we first find the top 10 most similar trajectories for each query in the candidate set. Then, for each query, trajectories in the candidate set are ranked according to their distance to the query in the embedding space. Similarly, if the trajectories ranking top 50 contain $k^\prime$ of the true top-10 neighbors, the R10@50 is $k^\prime/10$. These metrics can effectively evaluate whether the distance order in the embedding space is still preserved.

\textbf{Implementation Details.} 
We set the MLP output dimension in 1D preprocessing to 16. Intuitively, as the grid width $\delta$ decreases, ConvTraj will perform better, but the training cost of the model will also increase significantly. We thus set $\delta$ as 250 meters when generating binary images. For the Geolife dataset, the number of residual blocks for 1D convolution is $n=12(\lfloor \log_2 7579 \rfloor)$, Porto is $n=11(\lfloor \log_2 3836 \rfloor)$, and TrajCl-Porto is $n=7(\lfloor \log_2 200 \rfloor)$.
During training, we set the batch size to 128, the learning rate to $0.001$, and the embedding dimension to 128.
We evaluated four common trajectory similarity measurements, Hausdorff, DFD, DTW, and EDR on Geolife and Porto, and evaluated Hausdorff, DFD, EDwP~\cite{DBLP:conf/icde/RanuPTDR15}, and EDR on TrajCL-Porto. 
For each measurement on Geolife and Porto, we select three random seeds to repeat the experiment and report the average result and variance.
All experiments are conducted on a machine equipped with 36 CPU cores (Intel Core i9-10980XE CPU with 3.00GHz), 256 GB RAM, and a GeForce RTX 3090Ti GPU.

\begin{table*}[htb]
\caption{Efficiency Comparison}
\label{tab:efficiency}
\resizebox{0.95\textwidth}{!}{%
\begin{tabular}{c|c|cc|cc|cc|cc}
\toprule
& &\multicolumn{4}{c|}{\textbf{Geolife}}   & \multicolumn{4}{c}{\textbf{Porto}}   \\
\cline{3-10}
                    Method & $\#$ Paras
                    & \makecell{Pre-trained time \\ $t_{epoch}$ * ($\#$ epoch)}
                    & \makecell{Train time \\ $t_{epoch}$ * ($\#$ epoch)}
                    &\makecell{Train time \\ Per Epoch}
                    &\makecell{Inference \\ time}
                    &\makecell{Pre-trained time \\ $t_{epoch}$ * ($\#$ epoch)}
                    & \makecell{Train time \\$t_{epoch}$ * ($\#$ epoch)}
                    &\makecell{Train time \\ Per Epoch}
                    &\makecell{Inference \\ time}\\ \midrule

t2vec         &$2.86$M        &  $17.97$s * $10$  &  $0.27$s * $200$ & $18.24$s   &   $0.89$s     &   $328.12$s * $10$    &    $0.27$s * $200$   & $328.39$s  &  $61.64$s           \\    
TrjSR         & $\approx 40000$   &   $273.05$s * $3$   &  $0.27$s * $200$   &  $273.33$s & $\textbf{0.09}$s    &   $11800$s * $3$    &   $0.27$s * $200$      & $11800.27$s  &  $\textbf{11.69}$s    \\  
TrajCL       &$5.49$M               &   $14.03$s * $54$   &   $145.73$s * $30$   & $159.76$s & $11.42$s       &   $208.73$s * $75$     &   $52.14$s * $30$         & $260.87$s  &  $367.12$s      \\ 
NeuTraj        &$0.10$M         &    -    &    $149.13$s * $100$ & $149.13$s   &  $41.48$s       &   -     &   $230.29$s * $100$ & $230.29$s          &   $832.58$s      \\ 
TrajGAT            &$11.45$M       & -    &  $2613$s * $50$   & $2613$s&   $257.49$s    &   - &     $1843$s * $50$      &$1843$s   &  $4946.38$s   \\
\hline
ConvTraj &$0.13$M     &     -    &   $1.57$s * $200$  & $\textbf{1.57}$s  &    $0.41$s       &    -     &    $1.07$s * $200$       & $\textbf{1.07}$s    &   $28.53$s    \\ 
\bottomrule
\end{tabular}
}
\end{table*}

\subsection{Effectiveness}
\autoref{tab:metric-result-haus-dfd}, ~\autoref{tab:metric-result-dtw-edr}, and \autoref{tab:metric-result-haus-dfd-chengdu}
present an overview of the performance exhibited by different methods concerning the top-$k$ similarity search task on Geolife, Porto, and Chengdu, we can observe that: (1) On all datasets, ConvTraj significantly outperforms all methods on all metrics. Taking the Hausdorff distance on the Geolife as an example, compared with the state-of-the-art baseline NeuTraj, ConvTraj exceeds by more than 11\% in all metrics, with the largest improvement of 15.42\% for HR@5 and the smallest improvement of 11.64\% for HR@1. In addition, even for the Porto which contains 1.6 million trajectories, R10@50 has at least a 10.75\% improvement on four measurements.
This non-negligible improvement in performance is impressive given the fact that the sequence order features extracted by 1D convolution and the geographical distribution of the trajectory extracted by 2D convolution are both very beneficial to generating high-quality trajectory embedding representations. (2) The advantage of ConvTraj is evident in all measurements, which shows that ConvTraj is a general framework for different measurements. We can observe that no method can handle all measurements well. For example, NeuTraj performs best on the Hausdorff and DFD, while TrjSR and TrajCL have advantages on DTW and EDR respectively, which is also mutually verified with the results in ~\cite{DBLP:journals/corr/abs-2311-00960}. However, ConvTraj achieves state-of-the-art accuracy in all measurements. Compared to the state-of-the-art, ConvTraj achieves an average improvement of 10.22\%, 8.02\%, 7.59\%, and 7.03\% on all metrics of the Hausdorff, DFD, DTW, and EDR in Porto respectively. (3) We also noticed that compared with the results on the Geolife and Porto datasets, the TrajGAT method performed better on the Chengdu dataset. This may be because the longitude and latitude of the Chengdu dataset cover a larger area, so the quadtree-based modeling method of the TrajGAT is more effective.

\autoref{tab:trajcl-result-haus-dfd} presents the experimental results on TrajCL-Porto, we can observe that: (1) Similar to its performance on the Geolife and Porto, the ConvTraj method surpasses state-of-the-art in almost all metrics for four measurements. Compared with the state-of-the-art, Convtraj achieves improvements of 12\%, 23\%, 6.8\%, and 14.2\% on the HR@5 metrics of EDR, EDwP, Hausdorff, and DFD. (2) Even though both were tested on the Porto dataset, the performance gap between \autoref{tab:metric-result-dtw-edr} and \autoref{tab:trajcl-result-haus-dfd} is very large. For example, the HR@5 of the TrajCL and ConvTraj in \autoref{tab:trajcl-result-haus-dfd} on the DFD are 0.618 and 0.760 respectively, but in \autoref{tab:metric-result-dtw-edr} they are 0.141 and 0.349 respectively. The reason is that the TrajCL-Porto dataset contains fewer trajectories. When performing the top-$k$ similarity search task, the TrajCL-Porto dataset only has 2000 candidate trajectories. However, the Porto used in \autoref{tab:metric-result-haus-dfd} and \autoref{tab:metric-result-dtw-edr} contains 1598079 candidate trajectories, which results in a more comprehensive result. (3) We also evaluate the performance of ConvTraj using only 1D convolution (ConvTraj-1D CNN) or 2D convolution (ConvTraj-2D CNN) on TrajCL-Porto, and we can observe that ConvTraj's performance degrades after missing some features, but still has excellent performance.

\begin{table*}[htb]
\caption{Ablation Studies Results: The Role of 1D and 2D Convolution}
\label{tab:ablation-result}
\resizebox{0.95\textwidth}{!}{%
\begin{tabular}{c|c|ccc|ccc|ccc|ccc}
\toprule
& &\multicolumn{3}{c|}{\textbf{Haus}}   & \multicolumn{3}{c|}{\textbf{DFD}} & \multicolumn{3}{c}{\textbf{DTW}} & \multicolumn{3}{c}{\textbf{EDR}}   \\
\cline{3-14}
                     & Method & HR@10 & HR@50 & R10@50 & HR@10 & HR@50 & R10@50& HR@10 & HR@50 & R10@50& HR@10 & HR@50 & R10@50\\ \midrule
\multirow{3}{*}{Geolife}  & 1D CNN &  $38.89$     &    $56.90$   &     $79.01$ &  $62.51$     &    $76.44$   &     $95.50$&  $32.02$     &    $45.82$   &     $68.34$ &  $11.32$     &    $14.60$   &     $16.09$\\
                         & 2D CNN &  $57.97$     &    $72.11$   &     $92.49 $ &  $44.28$     &    $54.32$   &     $85.37 $ &  $35.46$     &    $45.19$   &     $75.22$&  $22.41$     &    $26.59$   &     $50.00$\\
                         & \textbf{1D + 2D} &    $\textbf{63.69}$     &    $\textbf{76.12}$   &     $\textbf{95.20}$&    $\textbf{68.86}$     &    $\textbf{79.52}$   &     $\textbf{97.34}$ &   $\textbf{46.46}$     &    $\textbf{59.26}$   &     $\textbf{83.70}$&    $\textbf{28.64}$     &    $\textbf{30.75}$   &     $\textbf{54.93}$\\
\hline
\multirow{3}{*}{Porto}  & 1D CNN &  $13.80$     &    $25.53$   &     $39.68$&  $10.30$     &    $17.86$   &     $32.02$&  $3.16$     &    $8.24$   &     $13.06$&  $12.58$     &    $14.45$   &     $24.86$\\
                         & 2D CNN  &  $28.46$     &    $40.72$   &     $60.04$&  $24.52$     &    $35.35$   &     $54.84$ &  $19.88$     &    $30.39$   &     $34.91 $&  $9.90$     &    $17.26$   &     $26.96$ \\
                         & \textbf{1D + 2D}  &    $\textbf{33.27}$     &    $\textbf{45.98}$   &     $\textbf{67.20}$&    $\textbf{40.59}$     &    $\textbf{53.35}$   &     $\textbf{77.33}$&   $\textbf{24.46}$     &    $\textbf{35.36}$   &     $\textbf{52.83}$&    $\textbf{15.61}$     &    $\textbf{21.45}$   &     $\textbf{37.00}$\\
\bottomrule
\end{tabular}
}
\end{table*}

\begin{table*}[htb]
\caption{Ablation Studies Results: Use LSTM to Replace 1D Convolution}
\label{tab:ablation-result-len}
\resizebox{0.95\textwidth}{!}{%
\begin{tabular}{c|c|ccc|ccc|ccc|ccc}
\toprule
& &\multicolumn{3}{c|}{\textbf{Hausdorff}}   & \multicolumn{3}{c|}{\textbf{DFD}} & \multicolumn{3}{c}{\textbf{DTW}} & \multicolumn{3}{c}{\textbf{EDR}}    \\
\cline{3-14}
                & Method & HR@10 & HR@50 & R10@50 & HR@10 & HR@50 & R10@50 & HR@10 & HR@50 & R10@50 & HR@10 & HR@50 & R10@50\\ \midrule
\multirow{3}{*}{Porto-S-10}  & 2D CNN&  $72.92$     &    $84.46$   &     $99.56$   &  $64.56$     &    $78.03$   &     $98.58$ &  $68.32$     &    $84.74$   &     $99.02$ &  $46.36$     &    $44.57$   &     $83.90$\\
                         & LSTM+2D&  $79.40$     &    $88.36$   &     $99.76$&  $82.92$     &    $88.77$   &     $\textbf{99.90}$ &  $\textbf{86.60}$     &    $\textbf{95.36}$   &     $\textbf{99.66}$ &  $\textbf{47.26}$     &    $\textbf{45.56}$   &     $84.56$ \\
                         & \textbf{1D+2D} &    $\textbf{80.40}$     &    $\textbf{88.64}$   &     $\textbf{99.78}$ &    $\textbf{83.02}$     &    $\textbf{88.84}$   &     $99.82$&    $86.20$     &    $95.35$   &     $\textbf{99.66}$ &    $47.24$     &    $45.51$   &     $\textbf{85.20}$\\
\hline
\multirow{3}{*}{Porto-S-70}  & 2D CNN &  $68.84$     &    $77.74$   &     $98.72$&  $46.70$     &    $53.44$   &     $90.32$&  $52.26$     &    $61.10$   &     $92.54$&  $30.34$     &    $32.14$   &     $59.40$\\
                         & LSTM+2D &  $69.32$     &    $77.72$   &     $98.80$&  $64.22$     &    $71.64$   &     $97.18$&  $69.28$     &    $79.05$   &     $98.22$&$34.68$     &    $36.21$   &     $63.64$\\
                         & \textbf{1D+2D}  &    $\textbf{72.22}$     &    $\textbf{79.67}$   &     $\textbf{99.24}$&    $\textbf{72.64}$     &    $\textbf{78.60}$   &     $\textbf{98.88}$&    $\textbf{72.24}$     &    $\textbf{82.53}$   &     $\textbf{98.84}$&    $\textbf{38.16}$     &    $\textbf{38.08}$   &     $\textbf{66.78}$\\
\hline
\multirow{3}{*}{Porto-S}  & 2D CNN &  $63.54$     &    $72.57$   &     $97.36$&  $27.36$     &    $34.37$   &     $68.08$&  $36.40$     &    $43.60$   &     $78.64$&  $21.34$     &    $24.40$   &     $46.86$\\
                         & LSTM+2D &  $64.10$     &    $73.22$   &     $\textbf{97.44}$&  $27.32$     &    $34.31$   &     $68.18$&  $37.64$     &    $43.87$   &     $78.10$&  $22.02$     &    $24.55$   &     $47.70$\\
                         & \textbf{1D+2D} &    $\textbf{65.04}$     &    $\textbf{73.76}$   &     $97.20$ &    $\textbf{58.22}$     &    $\textbf{68.20}$   &     $\textbf{94.70}$&    $\textbf{57.42}$     &    $\textbf{67.64}$   &     $\textbf{94.44}$ &    $\textbf{25.62}$     &    $\textbf{28.82}$   &     $\textbf{53.14}$ \\
\bottomrule
\end{tabular}
}
\end{table*}

\subsection{Efficiency}
We evaluate the efficiency of all baselines with open-source implementations on Geolife and Porto, and report the results in ~\autoref{tab:efficiency}. This includes network parameters, training time, and inference time. For methods that require pre-training, we also report their pre-training time. As illustrated, compared to existing RNN-based and Transformer-based methods, ConvTraj not only has fewer parameters (only 0.03M more than NeuTraj) but also has great advantages in training and inference speed. Taking the Porto with 1.6 million items as an example, compared with the most efficient Transformer-based model TrajCL, the training speed per epoch and the inference speed of ConvTraj are at least 243.80x and 12.87x faster respectively. Compared with the most efficient RNN-based model t2vec, the training speed per epoch and the inference speed of ConvTraj are at least 306.91x and 2.16x faster respectively. The reason for such a huge improvement is that compared to Transformer-based methods, ConvTraj has fewer parameters. Meanwhile, compared with RNN-based methods, although the parameters of ConvTraj are relatively large, the training and inference of ConvTraj are more efficient due to the inherent low parallelism of RNN. In addition, we also note that: (1) Compared with the CNN-based TrjSR, ConvTraj has no advantage in inference, but the training is faster because TrjSR requires pre-training on a large number of trajectories, and only uses fewer convolutional layers during inference, which also shows the superiority of CNN in terms of efficiency. (2) Although both t2vec and NeuTraj are based on RNN, and NeuTraj has fewer parameters, t2vec is more efficient. The reason is that NeuTraj needs to select more triplets during the training phase and compute spatial attention based on adjacent grids at each time step.

\subsection{Ablation Studies}
\subsubsection{The Role of 1D and 2D Convolution}
Our ConvTraj combines 1D and 2D convolutions, we thus conducted the following experiments to evaluate the contributions of each module: (1) 1D CNN. Using only 1D convolution features. (2) 2D CNN. Using only 2D convolution features. (3) 1D+2D. Using 1D and 2D convolution together. The results in ~\autoref{tab:ablation-result} show that for all measurements, neglecting any of these modules leads to a reduction in performance. In addition, we observe that 2D CNN outperforms most baselines, including TrjSR, which also uses 2D convolution. A similar conclusion can also be derived from \autoref{tab:trajcl-result-haus-dfd}. We explain that the goal of TrjSR is to reconstruct a high-resolution image from a low-resolution so that it can be as close as possible to the original image, thus the backbone and loss used are quite different from our 2D CNN. Furthermore, although we fine-tuned TrjSR, our 2D CNN is trained end-to-end and thus has more advantages.

\subsubsection{Use LSTM to Replace 1D Convolution}
In our ConvTraj, the role of 1D convolution is to capture the sequential features of trajectories. Although RNNs are commonly used for this purpose~\cite{DBLP:conf/icde/YaoCZB19}, we aim to demonstrate the important role of 1D convolution in ConvTraj by replacing it with an LSTM network. We will compare three methods to show the effectiveness of 1D convolution in capturing sequential features: (1) 2D CNN. Only using 2D convolution. (2) LSTM+2D. Using LSTM and 2D convolution together. (3) 1D+2D. Using 1D and 2D convolution together. The dataset used in this study is the same as \autoref{table:motivating-example}, called \textbf{Porto-S}, and the number of GPS points contained in each trajectory in Porto-S ranges from 104 to 888. In addition, we generated two more datasets, \textbf{Porto-S-10} and \textbf{Porto-S-70}, which contain the first 10 and 70 GPS points of each trajectory in Porto-S respectively, i.e., each trajectory in Porto-S-10 contains 10 GPS points, and each trajectory in Porto-S-70 contains 70 GPS points.

As shown in \autoref{tab:ablation-result-len}, the performance of LSTM+2D and 1D+2D is similar in the Porto-S-10, and LSTM+2D even performs slightly better than 1D+2D at some measurements (e.g., DTW and EDR), and both methods are significantly outperform than 2D CNN. These results show that LSTM performs very well in capturing sequential features when the trajectory contains fewer GPS points. However, as the number of points in a trajectory increases, the ability of LSTM to capture sequential features gradually decreases. For example, the HR@1 of 1D+2D and LSTM+2D on Porto-S-10 are 62.40\% and 63.20\% respectively for DTW. However, the HR@1 on Porto-S-70 are 52.40\% and 46.40\% respectively, and on Porto-S they are 38.60\% and 22.20\%. The gaps between them are -0.80\%, +6.0\%, +16.4\%, increasing progressively. Even LSTM+2D performs nearly as well as 2D CNN alone on Porto-S. These results show that RNNs struggle to capture the sequential features of trajectories with a large number of GPS points, whereas 1D CNNs do not exhibit this limitation.

\begin{table}[htb]
\caption{Employ the image generation strategy (Grayscale Image, i.e., GI) of TrjSR On Geolife}
\label{tab:ablation-result-trjsr}
\resizebox{\columnwidth}{!}{%
\begin{tabular}{c|c|cccc}
\toprule
& Method                   & HR@1  & HR@10 & HR@50 & R10@50  \\\midrule
\multirow{3}{*}{Haus} & \multirow{1}{*}{Binary Image} & $\textbf{46.17}$        & $\textbf{63.69}$& $\textbf{76.12}$& $\textbf{95.20}$    \\
                         & \multirow{1}{*}{GI (2D Avg-Pooling)} & $36.10$        & $56.49$& $72.70$& $91.53$   \\
                         & \multirow{1}{*}{GI (2D Max-Pooling)} & $42.50$        & $59.90$& $72.92$& $92.52$   \\
\cline{1-6}
\multirow{3}{*}{DFD} & \multirow{1}{*}{Binary Image} & $\textbf{51.80}$        & $\textbf{68.86}$& $\textbf{79.52}$& $\textbf{97.34}$    \\
                         & \multirow{1}{*}{GI (2D Avg-Pooling)} & $41.20$        & $63.28$& $77.74$& $95.71$   \\
                         & \multirow{1}{*}{GI (2D Max-Pooling)} & $50.20$        & $66.89$& $78.71$& $96.71$   \\
\cline{1-6}
\multirow{3}{*}{DTW} & \multirow{1}{*}{Binary Image} & $\textbf{31.70}$        & $\textbf{46.46}$& $\textbf{59.26}$& $\textbf{83.70}$    \\
                         & \multirow{1}{*}{GI (2D Avg-Pooling)} & $29.90$        & $45.03$& $58.21$& $82.02$   \\
                         & \multirow{1}{*}{GI (2D Max-Pooling)} & $31.50$        & $44.63$& $57.19$& $82.49$   \\
\cline{1-6}
\multirow{3}{*}{EDR} & \multirow{1}{*}{Binary Image} & $25.96$        & $28.64$& $30.75$& $54.93$    \\
                         & \multirow{1}{*}{GI (2D Avg-Pooling)} & $22.50$        & $25.44$& $28.07$& $50.59$   \\
                         & \multirow{1}{*}{GI (2D Max-Pooling)} & $\textbf{28.80}$        & $\textbf{33.25}$& $\textbf{33.45}$& $\textbf{62.45}$   \\
\bottomrule
\end{tabular}%
}
\end{table}

\begin{table}[htb]
\caption{Sensitivity of Embedding Results to $\delta$ On Geolife}
\label{tab:ablation-result-delta}
\resizebox{\columnwidth}{!}{%
\begin{tabular}{c|cc|cccc}
\toprule
&\makecell{Train time \\ Per Epoch}        & Method        & HR@1  & HR@10 & HR@50 & R10@50  \\\midrule
\multirow{3}{*}{Haus} & \multirow{1}{*}{9.86s} & \multirow{1}{*}{$\delta$ = 125m} & $\textbf{49.70}$        & $63.60$& $\textbf{76.58}$& $\textbf{95.85}$    \\
                         & \multirow{1}{*}{1.57s}& \multirow{1}{*}{$\delta$ = 250m} & $46.17$        & $\textbf{63.69}$& $76.12$& $95.20$   \\
                         & \multirow{1}{*}{1.19s}& \multirow{1}{*}{$\delta$ = 375m} & $45.20$        & $61.23$& $75.35$& $94.66$    \\
\cline{1-7}
\multirow{3}{*}{DFD} & \multirow{1}{*}{9.86s}& \multirow{1}{*}{$\delta$ = 125m} & $\textbf{52.70}$        & $\textbf{69.10}$& $\textbf{79.69}$& $\textbf{97.37}$    \\
                         & \multirow{1}{*}{1.57s}& \multirow{1}{*}{$\delta$ = 250m} & $51.80$        & $68.86$& $79.52$& $97.34$   \\
                         & \multirow{1}{*}{1.09s}& \multirow{1}{*}{$\delta$ = 375m} & $51.40$        & $67.15$& $79.00$& $97.15$    \\
\cline{1-7}
\multirow{3}{*}{DTW}& \multirow{1}{*}{9.86s} & \multirow{1}{*}{$\delta$ = 125m} & $\textbf{35.10}$        & $\textbf{47.04}$& $\textbf{59.65}$& $\textbf{84.86}$    \\
                         & \multirow{1}{*}{1.57s}& \multirow{1}{*}{$\delta$ = 250m} & $31.70$        & $46.46$& $59.26$& $83.70$   \\
                         & \multirow{1}{*}{1.09s}& \multirow{1}{*}{$\delta$ = 375m} & $30.60$        & $44.54$& $57.78$& $81.92$    \\
\cline{1-7}
\multirow{3}{*}{EDR} & \multirow{1}{*}{9.86s}& \multirow{1}{*}{$\delta$ = 125m} & $\textbf{28.70}$        & $\textbf{33.76}$& $\textbf{33.74}$& $\textbf{62.85}$    \\
                         & \multirow{1}{*}{1.57s}& \multirow{1}{*}{$\delta$ = 250m} & $25.96$        & $28.64$& $30.75$& $54.93$   \\
                         & \multirow{1}{*}{1.09s}& \multirow{1}{*}{$\delta$ = 375m} & $24.50$        & $27.65$& $29.23$& $54.93$    \\
\bottomrule
\end{tabular}%
}
\end{table}

\subsubsection{2D Image Construction}
In 2D convolution, we map the trajectory into a binary image. In this section, we conducted some ablation experiments using the image generation strategy in TrjSR. As shown in \autoref{tab:ablation-result-trjsr}, we can see that there is a degradation in the performance of the model when directly applying the image generation strategy used in TrjSR (GI with 2D Avg-Pooling) to ConvTraj. However, after replacing the average pooling in the 2D convolution with the max-pooling, its performance is close to ConvTraj and performs better in EDR distance. We explain that the reason is, for grayscale images of TrjSR, the more points imply the longer duration that the object stays in this grid. Thus, different pixel values can be used to capture the temporal property of the trajectory. However, using the same average pooling as in ConvTraj on this grayscale image will generate additional noise, which affects the model’s ability to capture the geographical distribution. At this time, using a max-pooling layer can extract the strongest features in the grayscale image.

In addition, we also tested the sensitivity of our ConvTraj performance to the hyperparameter width $\delta$. As shown in \autoref{tab:ablation-result-delta}, we can see that with the decreasing of $\delta$, the performance of ConvTraj on all four distance functions has improved. This phenomenon is consistent with expectations, more points fall into different grids as $\delta$ decreases, thereby increasing the resolution of the 2D images. However, as the resolution increases, we can see that the training cost of the model also increases significantly. We thus use 250m as the default parameter in the paper.

\subsubsection{Loss Function Ablation Studies} 
In our ConvTraj, the loss function is the combination of triplet loss and MSE loss, where the role of MSE loss is to scale and approximate the trajectory distance, and triplet loss is to capture the relative similarity between trajectories. In order to demonstrate the role of these two loss functions in training, we conducted an ablation study on these two functions on the Geolife dataset. As shown in \autoref{tab:ablation-result-loss}, we can clearly observe that after removing the triplet loss, all metrics have declined. However, after removing the MSE loss, the metrics of Hausdorff and DFD have declined, while the metrics of DTW and EDR have increased. We guess that the reason for this problem is that compared to Hausdorff and DFD, the range of distance values of DTW and EDR is relatively large. As shown in \autoref{tab:distancevalue}, the distance value ranges of Hausdorff and DFD are between 0\textasciitilde0.32 and 0\textasciitilde0.34 respectively, however, the value ranges of DTW and EDR are between 0\textasciitilde1289.4 and 0\textasciitilde6256.0. Such a huge gap may cause the MSE loss to encounter some problems during scaling. A more detailed discussion may be studied in future work.

\begin{table}[htb]
\caption{Loss Function Ablation Studies Results On Geolife}
\label{tab:ablation-result-loss}
\resizebox{\columnwidth}{!}{%
\begin{tabular}{c|c|cccc}
\toprule
& Method                   & HR@1  & HR@10 & HR@50 & R10@50  \\\midrule
\multirow{3}{*}{Haus} & \multirow{1}{*}{w/o Triplet} & $44.30$        & $62.15$& $74.11$& $93.95$    \\
                         & \multirow{1}{*}{w/o MSE} & $45.40$        & $62.07$& $73.77$& $93.86$   \\
                         & \multirow{1}{*}{Triplet + MSE} & $\textbf{46.17}$        & $\textbf{63.69}$& $\textbf{76.12}$& $\textbf{95.20}$    \\
\cline{1-6}
\multirow{3}{*}{DFD} & \multirow{1}{*}{w/o Triplet} & $49.60$        & $67.71$& $79.02$& $97.02$    \\
                         & \multirow{1}{*}{w/o MSE} & $49.80$        & $68.11$& $78.81$& $96.98$   \\
                         & \multirow{1}{*}{Triplet + MSE} & $\textbf{51.80}$        & $\textbf{68.86}$& $\textbf{79.52}$& $\textbf{97.34}$    \\
\cline{1-6}
\multirow{3}{*}{DTW} & \multirow{1}{*}{w/o Triplet} & $23.00$        & $31.00$& $40.89$& $59.21$    \\
                         & \multirow{1}{*}{w/o MSE} & $\textbf{34.00}$        & $\textbf{47.16}$& $58.05$& $83.67$   \\
                         & \multirow{1}{*}{Triplet + MSE} & $31.70$        & $46.46$& $\textbf{59.26}$& $\textbf{83.70}$    \\
\cline{1-6}
\multirow{3}{*}{EDR} & \multirow{1}{*}{w/o Triplet} & $25.80$        & $26.58$& $27.26$& $51.49$    \\
                         & \multirow{1}{*}{w/o MSE} & $\textbf{33.30}$        & $\textbf{38.11}$& $\textbf{37.55}$& $\textbf{66.66}$   \\
                         & \multirow{1}{*}{Triplet + MSE} & $25.96$        & $28.64$& $30.75$& $54.93$    \\
\bottomrule
\end{tabular}%
}
\end{table}

\begin{table}[h]
  \centering
  \caption{Trajectory Distance Value Range}
  \label{tab:distancevalue}
  \resizebox{\columnwidth}{!}{%
  \begin{tabular}{ccccc}
    \toprule
    Measurements  & Hausdorff & DFD & DTW & EDR \\ \midrule
    Distance Value Range     & 0\textasciitilde0.32 & 0\textasciitilde0.34 & 0\textasciitilde1289.4 & 0\textasciitilde6256.0  \\
    \bottomrule
\end{tabular}
}
\end{table}

\begin{figure*}[h]
  \centering
  \includegraphics[width=\textwidth]{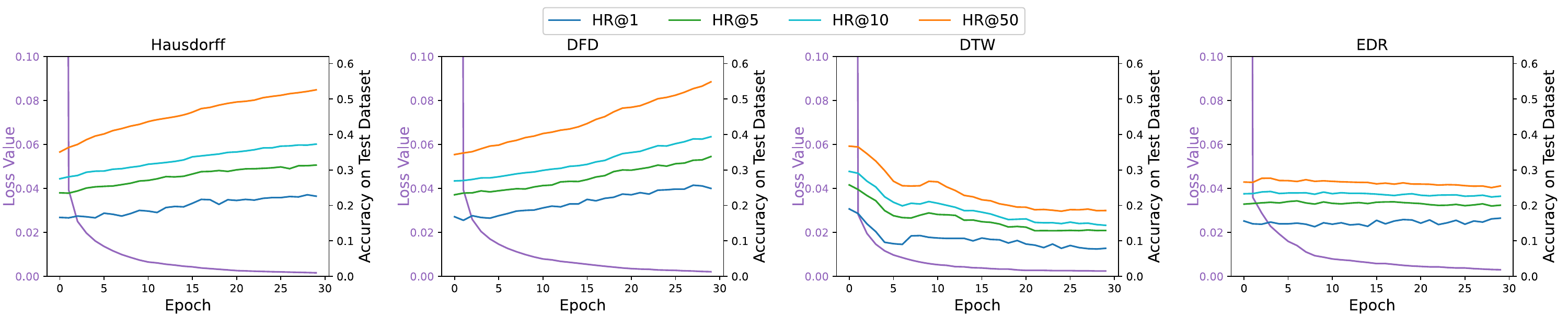}
  \caption{Training details of the TrajCL baseline training for 30 epochs. In this figure, the ordinate on the left represents the loss value of the model on the \textbf{training set}, and the ordinate on the right represents the performance of the model on the \textbf{test set} during the training process.}
  \label{fig:trajcl30}
\end{figure*}

\begin{figure*}[h]
  \centering
  \includegraphics[width=\textwidth]{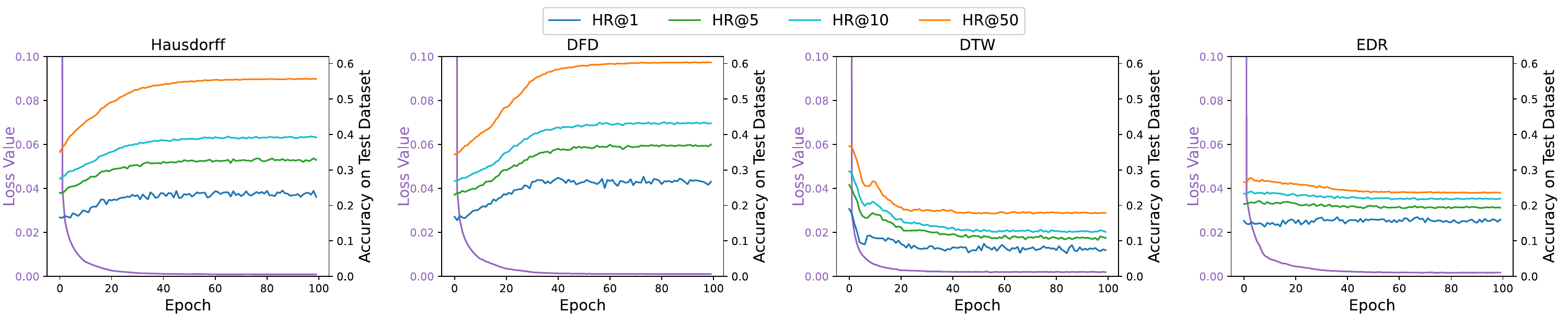}
  \caption{Training details of the TrajCL baseline training for 100 epochs}
  \label{fig:trajcl100}
\end{figure*}

\subsection{Training and Convergence Discussion}
\subsubsection{Motivation Experiment}

\begin{table}[htb]
\caption{The performance of the vanilla transformer under different training epochs}
\label{tab:more-epoch}
\resizebox{\columnwidth}{!}{%
\begin{tabular}{c|cc|cccc}
\toprule
Method & \makecell{Train time \\ Per Epoch}& \# Epoch                  & HR@1  & HR@5 & HR@10 & HR@50  \\\midrule
\multirow{3}{*}{global attention} & \multirow{1}{*}{17.28s} & \multirow{1}{*}{200} & $11.53$        & $21.49$& $26.46$& $35.85$    \\
                         & \multirow{1}{*}{17.28s} & \multirow{1}{*}{1000} & $22.10$ &$32.58$        & $39.11$& $50.11$   \\
                         & \multirow{1}{*}{17.28s} & \multirow{1}{*}{2000} & $23.10$ & $32.52$       & $41.59$& $55.15$   \\
\cline{1-7}
\multirow{3}{*}{\makecell{local attention \\ (w=10)}} & \multirow{1}{*}{17.28s} & \multirow{1}{*}{200} & $16.60$        & $26.47$& $31.89$& $41.40$    \\
                         & \multirow{1}{*}{17.28s} & \multirow{1}{*}{1000} & $23.20$ & $36.74$        & $42.80$& $54.70$   \\
                         & \multirow{1}{*}{17.28s} & \multirow{1}{*}{2000} & $22.80$ & $36.48$        & $43.89$& $57.34$   \\
\cline{1-7}
\multirow{3}{*}{\makecell{local attention \\ (w=5)}} & \multirow{1}{*}{17.28s} & \multirow{1}{*}{200} & $17.07$        & $26.75$& $31.82$& $41.53$    \\
                         & \multirow{1}{*}{17.28s} & \multirow{1}{*}{1000}& $21.80$& $35.40$        & $41.83$& $54.42$  \\
                         & \multirow{1}{*}{17.28s} & \multirow{1}{*}{2000} & $21.70$ & $34.36$        & $42.45$& $55.93$\\
\cline{1-7}
\multirow{1}{*}{1D CNN} & \multirow{1}{*}{1.03s} & \multirow{1}{*}{200}& $\textbf{33.23}$        & $\textbf{43.94}$& $\textbf{50.84}$& $\textbf{64.78}$    \\
\bottomrule
\end{tabular}%
}
\end{table}

Since the Transformer-based model has more parameters, we thus tested 200, 1000, and 2000 training epochs respectively. As shown in \autoref{tab:more-epoch}, we can observe that when the training epochs are increased from 200 to 1000, the Transformer-based model has a significant performance improvement (For example, the HR@1 of the method based on global attention increased from 11.53\% to 22.10\%). When the training epochs are increased from 1000 to 2000, the performance of the model does not change much, which indicates that the model training has reached convergence after 1000 epochs of training. However, we notice that even if the vanilla Transformer-based model reached convergence, it did not show an advantage over the 1D CNN that was only trained for 200 epochs. In addition, the overall training cost of the Transformer-based model was very high due to its larger number of parameters.

\begin{table}[htb]
\caption{The performance of the vanilla transformer under different training data}
\label{tab:more-data}
\resizebox{\columnwidth}{!}{%
\begin{tabular}{c|cc|cccc}
\toprule
Method & \makecell{Training \\ Data}& \# Epochs                  & HR@1  & HR@5 & HR@10 & HR@50  \\\midrule
\multirow{3}{*}{global attention} & \multirow{1}{*}{3000} & \multirow{1}{*}{2000} & $23.10$ & $32.52$       & $41.59$& $55.15$    \\
                         & \multirow{1}{*}{6000} & \multirow{1}{*}{2000} & $24.50$ & $35.02$       & $41.88$& $57.34$   \\
                         & \multirow{1}{*}{10000} & \multirow{1}{*}{2000} & $25.40$ & $36.22$       & $42.14$& $57.52$   \\
\cline{1-7}
\multirow{3}{*}{\makecell{local attention \\ (w=10)}} & \multirow{1}{*}{3000} & \multirow{1}{*}{2000} & $22.80$ & $36.48$        & $43.89$& $57.34$    \\
                         & \multirow{1}{*}{6000} & \multirow{1}{*}{2000} & $30.10$ & $41.38$        & $48.91$& $58.75$   \\
                         & \multirow{1}{*}{10000} & \multirow{1}{*}{2000} & $31.30$ & $42.22$        & $48.45$& $59.31$   \\
\cline{1-7}
\multirow{3}{*}{\makecell{local attention \\ (w=5)}} & \multirow{1}{*}{3000} & \multirow{1}{*}{2000} & $21.70$ & $34.36$        & $42.45$& $55.93$    \\
                         & \multirow{1}{*}{6000} & \multirow{1}{*}{2000} & $30.30$ & $41.58$        & $46.72$& $57.48$\\
                         & \multirow{1}{*}{10000} & \multirow{1}{*}{2000} & $30.80$ & $41.78$        & $47.95$& $57.42$\\
\cline{1-7}
\multirow{3}{*}{1D CNN} & \multirow{1}{*}{3000} & \multirow{1}{*}{200}& $33.23$        & $43.94$& $50.84$& $64.78$    \\
                        & \multirow{1}{*}{6000} & \multirow{1}{*}{200}& $34.40$        & $47.06$& $53.95$& $65.82$    \\
                        & \multirow{1}{*}{10000} & \multirow{1}{*}{200}& $\textbf{36.30}$        & $\textbf{48.40}$& $\textbf{55.46}$& $\textbf{67.90}$    \\
\bottomrule
\end{tabular}%
}
\end{table}

We also evaluated the performance of each model after increasing the number of trajectories used for training from 3000 to 6000, and 10000, and its results are shown in \autoref{tab:more-data}. We can see that the performance of the Transformer-based models has improved after increasing the number of trajectories used for training from 3000 to 6000 (For example, the HR@1 of the method based on global attention increased from 23.10\% to 24.50\%), but increasing the training data to 10000 does not significantly improve the performance. In addition, the performance of 1D CNN-based methods has also improved after the increase in training data. The vanilla Transformer-based method does not show a significant advantage when the training data is increased.

\begin{table}[]
\caption{The performance of the TrajCL}
\label{tab:trajcl}
\resizebox{\columnwidth}{!}{%
\begin{tabular}{c|c|cccc}
\toprule
Measurement & \# Epoch                  & HR@1  & HR@5 & HR@10 & HR@50  \\\midrule
\multirow{2}{*}{Hausdorff} & \multirow{1}{*}{30} & $22.03$ & $31.08$       & $37.21$& $52.49$    \\
                         & \multirow{1}{*}{100} & $\textbf{22.40}$ & $\textbf{32.82}$       & $\textbf{39.20}$& $\textbf{55.66}$   \\
\cline{1-6}
\multirow{2}{*}{DFD} & \multirow{1}{*}{30} & $25.40$ & $33.51$       & $38.98$& $54.72$    \\
                         & \multirow{1}{*}{100} & $\textbf{26.70}$ & $\textbf{37.16}$       & $\textbf{43.17}$& $\textbf{60.32}$   \\
\cline{1-6}
\multirow{2}{*}{DTW} & \multirow{1}{*}{30} & $\textbf{8.47}$ & $\textbf{12.48}$       & $\textbf{14.63}$& $\textbf{19.16}$    \\
                         & \multirow{1}{*}{100} & $7.40$ & $10.82$       & $12.54$& $17.88$   \\
\cline{1-6}
\multirow{2}{*}{EDR} & \multirow{1}{*}{30} & $\textbf{17.00}$ & $\textbf{20.33}$       & $\textbf{22.74}$& $\textbf{24.67}$    \\
                         & \multirow{1}{*}{100} & $16.00$ & $19.34$       & $21.91$& $23.58$   \\
\bottomrule
\end{tabular}%
}
\end{table}

\subsubsection{The Training Details of TrajCL}
Since the baseline of TrajCL performs best among all current Transformer-based baselines, we thus added more details about the training and convergence of TrajCL in the following. 

In our above experiments, in order to make the comparison as fair as possible, the report results based on the TrajCL are all based on its default open-source settings. For the Geolife dataset, we first pre-train TrajCL with 10000 trajectories and use the default training epoch of 100 in its open-source code. We then fine-tune TrajCL with the ground truth, with the default training epoch of 30 in its open-source code. As shown in \autoref{fig:trajcl30}, we can see that for Hausdorff and DFD, the model does not seem to have converged at this time, but for DTW and EDR distance, the model has overfitted.

We thus increased the number of training epochs from 30 to 100. The training details are shown in \autoref{fig:trajcl100}. We can observe that the model has converged after 100 epochs of training. We report the performance of TrajCL after 30 and 100 epochs of training in \autoref{tab:trajcl}. We can see that the performance of Hausdorff and DFD increases after 100 epochs of training, but the performance of DTW and EDR decreases after 100 epochs of training. In addition, this method of determining model parameters cannot be applied in practice because we cannot leak the test set during training, the example here is just for illustration.

\section{Conclusion}
In this paper, we argue that trajectory similarity learning should pay more attention to local similarity. Then we propose a simple CNN-based framework ConvTraj. Some theoretical analysis is conducted to help justify the effectiveness of ConvTraj. Extensive experiments on four real-world datasets show the superiority of ConvTraj.



\bibliographystyle{ACM-Reference-Format}
\bibliography{sample}

\end{document}